\documentclass[twoside]{article}

\usepackage[accepted]{aistats2020}
\usepackage[square]{natbib}

\usepackage{hyperref}       
\usepackage{url}            
\usepackage{booktabs}       
\usepackage{amsfonts}       
\usepackage{nicefrac}       
\usepackage{microtype}      
\usepackage{natbib}
\usepackage{soul,color}

\usepackage[symbol]{footmisc}

\usepackage{bm}             
\usepackage{graphicx}       
\usepackage[noend]{algpseudocode} 
\usepackage{caption}        
\usepackage{array}          
\usepackage{booktabs}       
\usepackage{pbox}           
\usepackage{subcaption}     
\usepackage{floatrow}
\newfloatcommand{capbtabbox}{table}[][\FBwidth]

\usepackage{lmodern}
\usepackage[scale=2]{ccicons}
\usepackage{color}
\usepackage[linewidth=1pt]{mdframed}
\usepackage{cases}
\usepackage{amsfonts,amsmath,amssymb,amsthm}
\DeclareMathOperator*{\argmax}{arg\,max}
\DeclareMathOperator*{\argmin}{arg\,min}

\theoremstyle{plain}
\newtheorem{thm}{Theorem}
\newtheorem{lem}{Lemma}
\newtheorem{prop}{Proposition}

\theoremstyle{definition}
\newtheorem{defn}{Definition}

\theoremstyle{remark}

\usepackage{xcolor}
\usepackage[linesnumbered,ruled]{algorithm2e}

\SetCommentSty{mycommfont}
\SetKwInput{KwInput}{Input}                
\SetKwInput{KwOutput}{Output}              
\SetKwInput{KwInitialization}{Initialization}

\usepackage{colortbl}
\usepackage{color}

\begin{document}

%

%
\runningauthor{Thanh Tang Nguyen  \hspace{15 pt} Sunil Gupta \hspace{15 pt} Huong Ha \hspace{15 pt} Santu Rana \hspace{15 pt} Svetha Venkatesh}
\addtocounter{footnote}{-1}

\twocolumn[

\aistatstitle{Distributionally Robust Bayesian Quadrature Optimization}

\aistatsauthor{Thanh Tang Nguyen $^\dagger$  \hspace{15 pt} Sunil Gupta \hspace{15 pt} Huong Ha \hspace{15 pt} Santu Rana \hspace{15 pt} Svetha Venkatesh}

\aistatsaddress{Applied Artificial Intelligence Institute (A$^2$I$^2$) \\ Deakin University, Geelong, Australia} ]

\footnote{$^\dagger$ Corresponding: thanhnt@deakin.edu.au}


\begin{abstract}
Bayesian quadrature optimization (BQO) maximizes the expectation of an expensive black-box integrand taken over a known probability distribution. In this work, we study BQO under distributional uncertainty in which the underlying probability distribution is unknown except for a limited set of its i.i.d. samples. A standard BQO approach maximizes the Monte Carlo estimate of the true expected objective given the fixed sample set. Though Monte Carlo estimate is unbiased, it has high variance given a small set of samples; thus can result in a spurious objective function. We adopt the distributionally robust optimization perspective to this problem by maximizing the expected objective under the most adversarial distribution. In particular, we propose a novel posterior sampling based algorithm, namely distributionally robust BQO (DRBQO) for this purpose. We demonstrate the empirical effectiveness of our proposed framework in synthetic and real-world problems, and characterize its theoretical convergence via Bayesian regret. 

\end{abstract}

\section{Introduction}
\label{section:introduction}
Making robust decisions in the face of parameter uncertainty is critical to many real-world decision problems in machine learning, engineering and economics. {Besides the uncertainty that is inherent in data, a further difficulty arises due to the uncertainty in the context. A common example is hyperparameter selection of machine learning algorithms where cross-validation is performed using a small to medium sized validation set. Due to limited size of validation set, the variance across different folds might be high. 
Ignoring this uncertainty results in sub-optimal and non-robust decisions. {This problem in practice can be further exacerbated as the outcome measurements may be noisy and the black-box function itself is expensive to evaluate.} Being risk-averse is critical in such settings. 

{One way to capture the uncertainty in context is through a probability distribution. In this work, we consider stochastic black-box optimization that is distributionally robust to the uncertainty in context. We formulate the problem as}
\begin{align}
    \label{eq:stochastic_opt}
    \max_{x \in \mathcal{X} \subset \mathbb{R}^d} g(x) := \max_{x \in \mathcal{X}} \mathbb{E}_{
    P_0(w)} [f(x, w)],
\end{align}
where $f$ is an expensive black-box function and $P_0$ is a distribution over context $w$. We assume distributional uncertainty in which the distribution $P_0$ is known only through a {limited} set of its i.i.d samples  $S_n = \{w_1, ..., w_n\}$. This is equivalent to the scenario in which we are able to evaluate $f$ only on $\mathcal{X} \times S_n$ during optimization.



In the case that $P_0$ is known (e.g., $P_0$ is either available in an analytical form or easy to evaluate), a standard solution to the problem in Equation (\ref{eq:stochastic_opt}) is based on Bayesian quadrature \citep{o1991bayes,nipsRasmussenG02,oates2016probabilistic,OatesS19}. The main idea in this approach is that we {can} build a Gaussian Process (GP) model of $f$ and use the known relationship in the integral to imply a second GP model of $g$. This is possible because integration is a linear operator. 

Given the distributional uncertainty in which $P_0$ is only known through a limited set of its samples, a naive approach to the problem in Equation (\ref{eq:stochastic_opt}) is to maximize its Monte Carlo estimate:
\begin{align}
    \label{eq:mc}
    g_{mc}(x) :=  \mathbb{E}_{\hat{P}_n(w)}[f(x,w)],
\end{align}
where $\hat{P}_n(w) = \frac{1}{n} \sum_{i=1}^n \delta(w - w_i)$ and $\delta(.)$ is the Dirac distribution. When $n$ is sufficiently large, $g_{mc}(x)$ approximates $g(x)$ reasonably well as guaranteed by the weak law of large numbers; thus, the optimal solution of $g_{mc}(x)$ represents that of $g(x)$. In contrast, when $n$ is small, the optimal solution of $g_{mc}(x)$ might be sub-optimal to $g(x)$. {Since we are considering distributional perturbation, we cannot} guarantee the Monte Carlo estimate $g_{mc}(x)$ to be a good surrogate objective. 

A more {conservative} approach from statistical learning is to maximize the variance-regularized objective:
\begin{align}
    \label{eq:variance_reg}
    g_{bv}(x) := \mathbb{E}_{\hat{P}_n}[f(x, w)] - C_1 \sqrt{ Var_{\hat{P}_n}[ f(x, w) ] /n },
\end{align}
where $Var_{\hat{P}_n}$ denotes the empirical variance and $C_1$ is a constant determining the trade-off between bias and variance. {Thus}, given the context of {limited} samples, it is logical to use $g_{bv}(x)$ instead of $g_{mc}(x)$ as a surrogate objective for maximizing $g(x)$. 
However, unlike $g_{mc}(x)$, the variance term in $g_{bv}(x)$ breaks the linear relationship with respect to $f$. As a result, {though} $f$ is a GP, $g_{bv}(x)$ {need} not be \citep{o1991bayes}. 

\begin{figure}[t]
    \centering
    \includegraphics[scale=0.55]{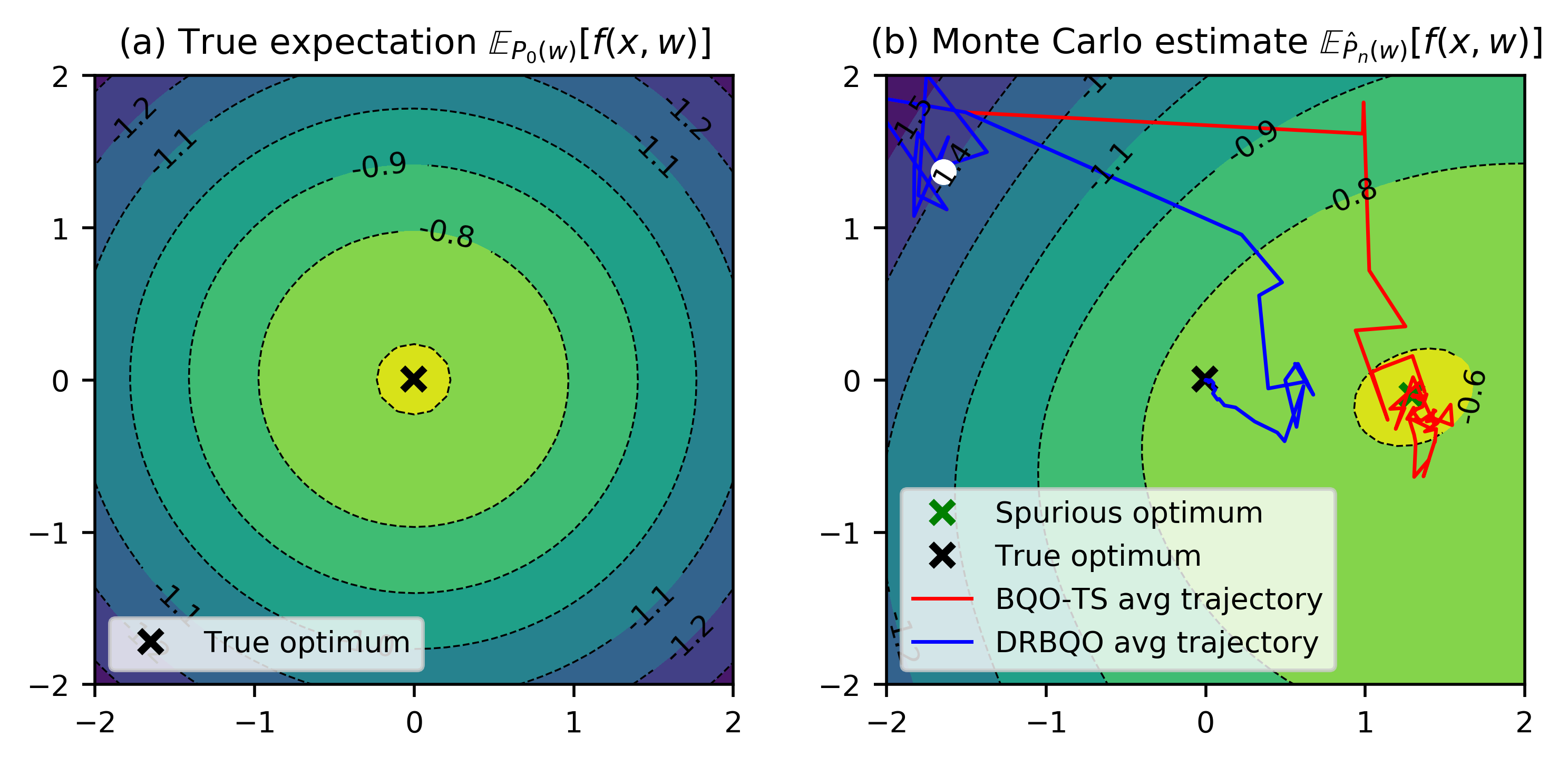}
    \caption{(a) The expected logistic function $g(x) = \mathbb{E}_{\mathcal{N}(w; 0, I)}[-\log(1 + e^{x^T w})]$ and (b) its Monte Carlo estimate using $10$ samples of $w$, and the averaged trajectories of our proposed algorithm DRBQO (detailed in Section \ref{section:alogirthm}) and a standard Bayesian Quadrature Optimization (BQO) baseline. Though being unbiased, Monte Carlo estimates can suffer from high variance given limited samples, resulting in spurious function estimates. Our proposed algorithm DRBQO approaches this mismatch problem by {finding} the distributionally robust solution under the most adversarial distribution over a $\chi^2$ distributional ball.}
    \label{fig:figure1}
\end{figure}
{Alternatively, we approach the distributional uncertainty problem above by formulating the distributionally robust Bayesian quadrature optimization}. In the face of the uncertainty about $P_0$, we seek to find a distributionally robust solution under the most adversarial distribution. Our approach is based on solving a surrogate distributionally robust optimization problem generated by posterior sampling at each time step. The surrogate optimization is solved efficiently via bisection search through any optimization. We demonstrate the efficiency of our algorithm in both synthetic and real-world problems. Our contributions are:
\begin{itemize}
    
    \item Demonstrating the limitations of standard Bayesian quadrature optimization algorithms under distributional uncertainty (Section \ref{section:setup}), and introducing a new algorithm, namely DRBQO, that overcomes these limitations (Section \ref{section:alogirthm}); 
    \item  Introducing the concept of $\rho$-regret for measuring algorithmic performance in this formulation  (Section \ref{section:setup}), and characterizing the theoretical convergence of our proposed algorithm in sublinear Bayesian regret (Section \ref{section:theory}); 
    \item Demonstrating the efficiency of DRBQO in finding distributionally robust solutions in both synthetic and real-world problems (Section \ref{section:experiment}).
\end{itemize}

\section{Related Work}
\label{section:related}
Our work falls {in} the {area} of Bayesian quadrature optimization whose goal is to perform black-box global optimization of an expected objective {of the form} $\int f(x,w) P(w) dw$. This type of problems is known with various names such as optimization of integrated response functions \citep{Williams:2000:SDC:932015}, multi-task Bayesian optimization \citep{DBLP:conf/nips/SwerskySA13}, and optimization with expensive integrands \citep{Toscano_IntegralBO_18}. This direction approaches the problem by evaluating $f(x,w)$ at one or {several} values of $w$ given $x$. This ameliorates the need of evaluating $f(x,w)$ at all the values of $w$ and can outperform  methods that evaluate the full objective via numerical quadrature \citep{FrazierBOtut18,Toscano_IntegralBO_18}. {All the previous approaches assume the knowledge of the distribution in the expected function. The distinction of our formulation is that we are interested in the distributional uncertainty scenario in which the underlying distribution is unknown except its empirical estimate.} 


Our work also {shares}  similarity with the distributionally robust optimization (DRO) literature \citep{DRO_review}. This problem setup considers the parameter uncertainty in real-world decision making problems. The uncertainty may be due to limited data and noisy measurements. DRO takes into account this uncertainty and approaches the problem by taking the worst-case of the underlying distribution within an uncertainty set of distributions. DRO variants distinguish each other in design choices of the {distributional uncertainty set} and in problem contexts. Regarding the design of uncertainty sets, common designs specify the set of distributions with respect to the nominal distribution via  distributional discrepancy such as $\chi^2$ divergence \citep{NamkoongD16}, Wasserstein distance \citep{WasserstainDRO19}, and Maximum Mean Discrepancy \citep{MMD_DRO19}.  Regarding studying DRO in different problem contexts, the following contexts have been investigated: robust optimization \citep{DBLP:journals/mansci/Ben-TalHWMR13}, robust risk minimization \citep{NamkoongD16}, sub-modular maximization \citep{DBLP:conf/aistats/StaibWJ19}, boosting algorithms \citep{DROBoosting}, graphical models \citep{DBLP:conf/nips/FathonyRBZZ18}, games \citep{sun2018distributional,DBLP:conf/aistats/Zhu0WGY19}, fairness in machine learning \citep{DBLP:conf/icml/HashimotoSNL18}, Markov Decision Process \citep{nipsXuM10}, and reinforcement learning \citep{DRRL}. The distinction of our work is in terms of the problem context where we study DRO in Bayesian quadrature optimization.

\section{Problem Setup} 

\label{section:setup}
\textbf{Model}. Let $f: \mathcal{X} \times \Omega \rightarrow \mathbb{R}$ be an element of a reproducing kernel Hilbert space (RKHS) $\mathcal{H}_k$ where $k$: $\mathcal{X} \times \Omega \times \mathcal{X} \times \Omega \rightarrow \mathbb{R}$ is a positive-definite kernel, and $\mathcal{X}$ and $\Omega$ are, unless explicitly mentioned otherwise, compact domains in $\mathbb{R}^d$ and $\mathbb{R}^m$ for some {dimensions} $d$ and $m$, respectively. We further assume that $k$ is continuous and bounded from above by $1$, and that $\| f \|_k = \sqrt{\langle f, f \rangle_k } \leq B$ for some $B > 0$.  Two commonly used kernels are Squared Exponential (SE) and Mat\'ern \citep{RasmussenW06} which are similarly defined on $\mathcal{X} \times \Omega$ as follows: 
\begin{align*}
    &k_{SE}(.,.; ., .) = \exp ( - d^2_{\theta, \psi} (.,.; ., .)), \\
    &k_{\nu}(.,.; ., .) = \frac{2^{1-\nu}}{\Gamma(\nu)} \sqrt{2\nu} d_{\theta, \psi}(.,.; ., .) J_{\nu}(\sqrt{2\nu} d_{\theta, \psi}(.,.; ., .)), 
\end{align*}
where $\theta$ and $\psi$ are the length scales, $\nu >0$ defines the smoothness in the Mat\'ern kernel, $J(\nu)$ and $\Gamma(\nu)$ define the Bessel function and the gamma function, respectively, and $d^2_{\theta, \psi}(x, w; x', w') = \sum_{i=1}^{d} (x_i - x'_i)^2 / \theta^2_i +  \sum_{j=1}^{m}  (w_j - w'_j)^2 / \psi^2_j$. 

Let $P_0$ be a distribution on $\Omega$, and $S_n = \{w_1, ..., w_n\}$ be a fixed set of samples drawn from $P_0$. Though $f$ is defined on $\mathcal{X} \times \Omega$, we are interested in the distributional uncertainty scenario in which we can query $f$ only on $\mathcal{X} \times S_n$ during optimization. At time $t$, we query $f$ at $(x_t, w_t) \in \mathcal{X} \times S_n$ and observe a noisy reward $y_t = f(x_t, w_t) + \epsilon_t$, where $\epsilon_t \sim \mathcal{N}(0,\sigma^2)$. Our goal is to find a robust solution point $x \in \mathcal{X}$ such that $ \mathbb{E}_{P(w)}[f(x,w)]$ remains high even under the most adversarial realization of the unknown distribution $P_0$.  

Given a sequence of noisy observations $(x_i, w_i, y_i)_{i=1}^t$, the posterior distribution under a GP(0, $k(.,.,;.,.)$) prior is also also a GP with the following posterior mean and covariance: 
\begin{align*}
    \mu_t(x,w) &= k_t(x,w)^T(K_t + \sigma^2 I)^{-1} y_{1:t},  \\ 
    C_t(x,w; x', w') &= k(x,w; x', w') \\
    & - k_t(x,w)^T (K_t + \sigma^2 I)^{-1} k_t(x', w'), 
\end{align*}
where  $y_{1:t}=(y_1, ..., y_t)$, $k_t(x,w) = [k(x_i, w_i; x, w)]_{i=1}^t$, and $K_t = [k(x_i, w_i; x_j, w_j)]_{1 \leq i,j \leq t}$ is the kernel matrix.

We define the quadrature functional as
\begin{align}
    \label{eq:g_functional}
    g(f, x, P) := \int P(w|x) f(x,w) dw,  
\end{align}
for any conditional distribution $P(.|x)$ on $\Omega$ for all $x \in \mathcal{X}$, i.e., $P \in \mathcal{P}_{n,\rho} \times \mathcal{X}$. 
As an extended result of Bayesian quadrature \citep{o1991bayes}, for any conditional distribution $P \in \mathcal{P}_{n,\rho} \times \mathcal{X}$, $g(f,x,P)$ also follows a GP with the following mean and variance: 
\begin{align}
    \label{eq:quadrature_mu}
    &\mu_t(x, P) := \mathbb{E}_t[g(f,x, P)] = \int P(w|x) \mu_t(x,w) dw \\
    &\sigma^2_t(x, P) := Var_t[g(f, x, P)] \nonumber \\
    \label{eq:quadrature_sigma}
    &= \int \int P(w|x) P(w'|x) C_t(x,w; x, w') dw dw'. 
\end{align}
\textbf{Optimization goal}. We seek to optimize the expected function under the most adversarial distribution over some distributional uncertainty set $\mathcal{P}_{n,\rho} := \{ P | D(P, \hat{P}_n) \leq \rho \}$ : 
\begin{align}
    \label{eq:drbqo}
    \max_{x \in \mathcal{X}} \min_{P \in \mathcal{P}_{n,\rho}} \mathbb{E}_{P(w)} [f(x,w)], 
\end{align}
where $\hat{P}_n(w) = \frac{1}{n} \sum_{i=1}^n \delta(w-w_i)$ is the empirical distribution, $\rho \geq 0$ is the confidence radius around the empirical distribution with respect to a distribution divergence $D(.,.)$ such as Wasserstein distance, maximum mean discrepancy, and $\phi$-divergence.  We can interpret $\mathcal{P}_{n,\rho}$ as the set of {perturbed distributions with respect to} the empirical distribution $\hat{P}_n$ within a confidence radius $\rho$. We then seek a robust solution in the face of {adversarial distributional perturbation} within $\mathcal{P}_{n,\rho}$. 

{For any} distribution divergence choice $D$, we define a $\rho$-robust point to be any $x^*_{\rho}$ such that 
\begin{align}
    x^*_{\rho} \in \argmax_{x \in \mathcal{X}} \min_{P \in \mathcal{P}_{n,\rho}}  \mathbb{E}_{P(w)}[f(x,w)]. 
\end{align}
Our goal is to report after time $t$ a distributionally robust point $x_t$ in the sense that it has {small} $\rho$-\textbf{regret}, {which} is defined as
\begin{align}
    r_{\rho}(x)= g(f, x^*_{\rho}, P^*) - g(f, x, P^*),
    \label{eq:rho_regret}
\end{align}
where $P^*(.|x) = \argmin_{P \in \mathcal{P}_{n,\rho}} \sum_{w}P(w|x) f(x,w), \forall x$.

While our framework in this work can be adopted to various distribution divergences, we focus on the specific case when $D$ is $\chi^2$-divergence: 
$D(P, Q) = \frac{1}{2}\int_{\Omega} (\frac{dP}{dQ} -1)^2 dQ, \forall P,Q$. From here on, we {refer} $\mathcal{P}_{n,\rho}$ as the $\chi^2$ ball {with} $D$ being $\chi^2$-divergence. In particular, the distributionally robust optimization problem in Equation (\ref{eq:drbqo}) is equivalent to the variance-regularized optimization in Equation (\ref{eq:variance_reg}) when the variance is sufficiently high, as justified by the following theorem:
\begin{thm}[Modified from \cite{nipsNamkoongD17}]
Let $Z \in [M_0, M_1]$ be a random variable (e.g., $Z = f(x,w)$ for any fixed $x$), $\rho \geq 0$, $M = M_1 - M_0$, $s_n^2 = Var_{\hat{P}_n}[Z]$, $s^2 = Var[Z]$, and $OPT = \inf_{P} \left\{  \mathbb{E}_{P}[Z]:  P \in  \mathcal{P}_{n,\rho} \right\}$. Then $\max \left\{ \sqrt{2\rho s^2_n} - 2M\rho, 0 \right\}
    \leq \mathbb{E}_{\hat{P}_n}[Z] - OPT 
    \leq \sqrt{2\rho s^2_n}.$
Especially if $s^2 \geq \max\{ 24 \rho, \frac{16}{n},  \frac{1}{n s^2} \} M^2$, then $OPT = \mathbb{E}_{\hat{P}_n}[Z] -  \sqrt{2\rho s^2_n}$ 
with probability at least $1 - \exp(- \frac{n s^2}{36 M^2})$. 
\end{thm}
The intuition for this equivalence is that the $\chi^2$ ball and the variance penalty term in Equation (\ref{eq:variance_reg}) are both quadratic \citep{DBLP:conf/aistats/StaibWJ19}. Figure \ref{fig:2d_simplex} illustrates $\chi^2$ balls with various radii on the $3$-dimensional simplex. 

\begin{figure}[t]
    \centering
    \includegraphics[scale=0.25]{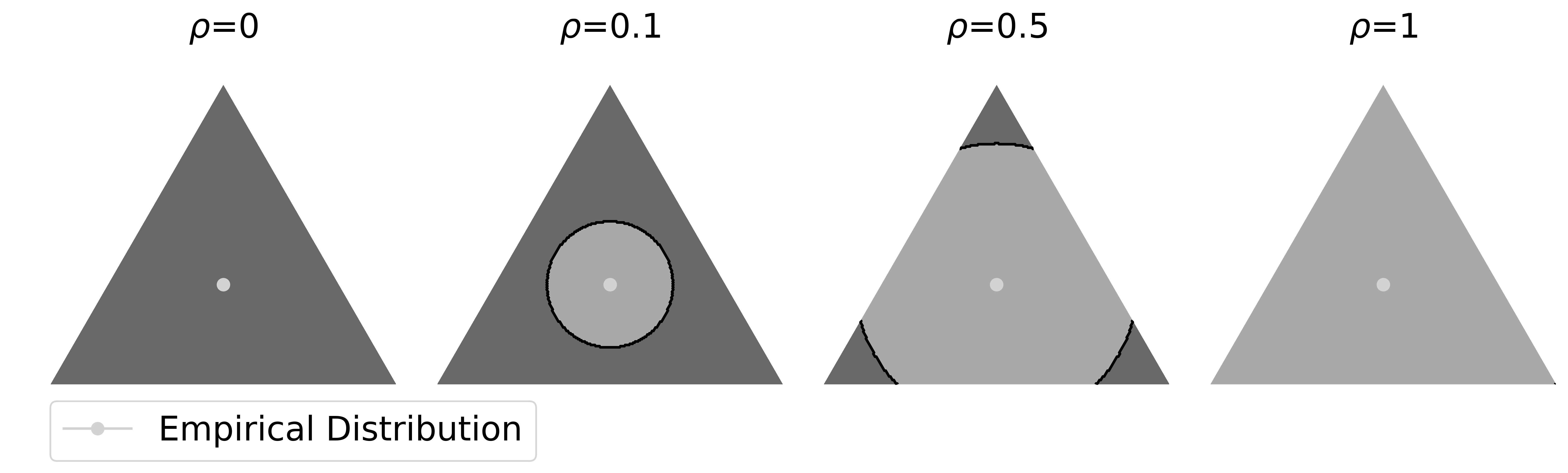}
    \caption{The $\chi^2$ balls with various radii $\rho$ on the $n$-dimensional simplex ($n=3$ in this example). The simplex, the $\chi^2$ balls and the empirical distribution are represented in dim gray, dark gray and light gray color, respectively. The $\chi^2$ ball with $\rho = 0$ reduces to a singleton containing only the empirical distribution while the ball becomes the entire simplex for $\rho \geq \frac{n-1}{2}$.}
    \label{fig:2d_simplex}
\end{figure}

\textbf{Failure of standard methods}.
Various methods have been developed for achieving small regret in maximizing $g(f,x, P_0) = \mathbb{E}_{P_0(w)} [f(x, w)]$ for some distribution $P_0(w|x) = P_0(w)$ \citep{Williams:2000:SDC:932015,DBLP:conf/nips/SwerskySA13,Toscano_IntegralBO_18}. These methods leverage the relationships in Equation (\ref{eq:quadrature_mu}) and (\ref{eq:quadrature_sigma}) to infer the posterior mean and variance of the expected function $g(f,x, P_0)$ from those of $f$. The inferred posterior mean and variance for $g(f,x, P_0)$ are then used in certain ways to acquire new points. While this is useful in the standard setting when we know $P_0$, {it is not useful when} we only have the empirical distribution $\hat{P}_n$. Specifically, an optimal solution found by these methods in the problem associated with the empirical distribution may be sub-optimal to that associated with the true distribution $P_0$. 

An illustrative example is depicted in Figure \ref{fig:figure1} where the averaged trajectories of our proposed DRBQO (detailed in Section \ref{section:alogirthm}) and a standard BQO baseline (detailed in Section \ref{section:experiment}) are also shown. Due to a limited number of samples of $P_0$, the Monte Carlo estimate $\mathbb{E}_{\hat{P}_n(w)}[f(x,w)]$ results in a spurious expected objective in this case. By resorting to the empirical distribution $\hat{P}_n$ constructed from the limited set of samples, the standard BQO baseline ignores the distributional uncertainty and converges to the optimum of the spurious expected objective. The same limitation applies to the standard BQO optimization methods, e.g., \cite{Williams:2000:SDC:932015,DBLP:conf/nips/SwerskySA13,Toscano_IntegralBO_18,DBLP:conf/wsc/PearceB17a} whose goal is to find a global non-robust maximum. 
\section{Algorithmic Approach}
\label{section:alogirthm}
Our main proposed algorithm is presented in Algorithm \ref{pseudo_code:drbqo}.
In the standard Bayesian quadrature problem in Equation (\ref{eq:mc}), we can easily adopt standard Bayesian optimization algorithms such as expected improvement (EI) \citep{mockus1978application} and an upper confidence bound (UCB) (e.g., GP-UCB \citep{DBLP:conf/icml/SrinivasKKS10}) using quadrature relationships in Equation (\ref{eq:quadrature_mu}) and (\ref{eq:quadrature_sigma}) \citep{DBLP:conf/nips/SwerskySA13}. However, like $g_{bv}(x)$ in Equation (\ref{eq:variance_reg}), $\min_{P \in \mathcal{P}_{n,\rho}} \mathbb{E}_{P(w)} [f(x,w)]$ {does not follow} a GP if $f$ follows a GP. This difficulty hinders the adoption of EI-like and UCB-like algorithms to our setting. 
We overcome this problem using posterior sampling \citep{DBLP:journals/mor/RussoR14}. 

\begin{algorithm}
\DontPrintSemicolon
\KwInput{Prior \texttt{GP}($\mu_0$, $k$), horizon $T$, fixed sample set $S_n$, confidence radius $\rho \geq 0, C_0 = k$.}

    \For{$t = 1$ to $T$}    
    { 
    
    
    \tcc{Posterior sampling}
            Sample $\tilde{f}_t \sim $ \texttt{GP}$\left(\mu_{t-1}, C_{t-1}\right)$. \; 
    \label{alg:line:ps}
            
    \tcc{A surrogate DR optimization}
        Choose $\displaystyle x_t \in \argmax_{x \in \mathcal{X}} \min_{P \in \mathcal{P}_{n,\rho}}  \mathbb{E}_{P} [\tilde{f}_t(x,w)]$. \; 
    \label{alg:line:dro}

    \tcc{Highest posterior variance}
    Choose $ \displaystyle w_t = \argmax_{w \in S_n} C_{t-1}(x_t, w; x_t, w)$. \; 
    \label{alg:line:w}

    Observe reward $\hat{y}_t \leftarrow f(x_t, w_t) + \epsilon_t$. \;
    
    Perform update \texttt{GP} to get $\mu_{t}$ and $C_{t}$.}

\KwOutput{
$\displaystyle \argmax_{ x \in \{x_1, ..., x_T\}} \min_{P \in \mathcal{P}_{n,\rho}}  \mathbb{E}_{P} [\mu_T(x,w)]$.
}
\caption{\texttt{DRBQO}: Distributionally Robust Bayesian quadrature optimization}
\label{pseudo_code:drbqo}
\end{algorithm}

The main idea of our algorithm is to sample and solve a surrogate distributionally robust optimization problem at each step guided by posterior sampling (lines \ref{alg:line:ps} and \ref{alg:line:dro} in Algorithm \ref{pseudo_code:drbqo}). In practice, we follow \cite{DBLP:conf/nips/Hernandez-LobatoHG14} to perform posterior sampling (line \ref{alg:line:ps} in Algorithm \ref{pseudo_code:drbqo}). 
Similar to the way posterior sampling is applied to standard Bayesian optimization problem \citep{DBLP:conf/nips/Hernandez-LobatoHG14}, a new point is selected according to the probability it is optimal in the sense of distributional robustness. One of the advantages of posterior sampling is that it avoids the need for confidence bound such as UCB. This is useful for our setting because the non-Gaussian nature of the distributionally robust objective makes it difficult to construct a deterministic confidence upper bound. 

Due to the convexity of the expectation with respect to a distribution, we can efficiently compute the value (therefore the gradients) of the inner minimization in line 8 of Algorithm \ref{pseudo_code:drbqo} in an analytical form via Lagrangian multipliers, {as presented in Proposition} \ref{prop:kkt}. 

\begin{prop}
\label{prop:kkt}
Let $l = (l_1, ..., l_n) \in \mathbb{R}^n$ (e.g., $l = \left( \tilde{f}_t(x,w_1), ..., \tilde{f}_t(x,w_n) \right)$ in line \ref{alg:line:dro} of Algorithm \ref{pseudo_code:drbqo}), $\hat{P}_n = (\frac{1}{n}, ..., \frac{1}{n})$ {being} the weights of the empirical distribution, $\Delta_n$ being the $n$-dimensional simplex,
    $\mathcal{P}_{n,\rho} = \bigg\{P \in \Delta_n  \bigg| \frac{1}{2}\int_{\Omega} (\frac{dP}{d\hat{P}_n} -1)^2 d\hat{P}_n \leq \rho \bigg\}$
being the $\chi^2$-ball around the empirical distribution with radius $\rho$. 
Then, the optimal weights $p = (p_1, ..., p_n) = \argmin_{q \in \mathcal{P}_{n,\rho}} q^T l$ satisfy the systems of {relations} with variables $(p, \lambda, \eta)$:
\begin{numcases}{}
\label{kkt:compute_pi}
\lambda p_i = \frac{1}{n} \max \{-l_i - \eta, 0\}, \forall 1 \leq i \leq n \nonumber \\
\label{kkt:compute_eta}
\eta |A| + n \lambda = -\sum_{i \in A} l_i \text{ where } A = \{i: l_i \leq -\eta \} \nonumber \\
\label{kkt:rho_equation}
\lambda \left(2\rho + 1 - n \| p\|_2^2 \right)  = 0 \\ 
n \| p\|_2^2 \leq 2\rho + 1, \text{ and } \gamma \geq 0. \nonumber 
\label{kkt:dual_1}
\end{numcases}

\end{prop}
\begin{proof}
The constrained minimization $\min_{p \in \mathcal{P}_{n,\rho}} p^T l$  is a convex optimization problem which forms the Lagrangian: $L(p, \lambda, \eta, \zeta) = p^Tl - \lambda \left(\rho - \frac{1}{2n} \sum_{i=0}^n (n p_i - 1)^2 \right)  -\eta (1 - \sum_{i=1}^n p_i) - \sum_{i=1}^n \zeta_i p_i$
where $p \in \mathbb{R}^n, \lambda \geq 0$, $\eta \in \mathbb{R}$, and $\zeta \in \mathbb{R}^n_{+}$. 

The system of linear equations in the proposition {emerges from}  Karush-Kuhn-Tucker (KKT) conditions and simple rearrangements. Note that since the primal problem is convex, the duality gap is zero and the KKT conditions are the sufficient and necessary conditions for the primal problem. 

Notice the first two equations 
that we can compute $p_i$ in terms of $\lambda$. These $p_i = p_i(\lambda)$ are then substituted into Equation (\ref{kkt:rho_equation}) to solve for $\lambda$. In practice, we can use bisection search \citep{NamkoongD16} to solve for $\lambda$ satisfying  Equation (\ref{kkt:rho_equation}). The details of this algorithm and of Proposition \ref{prop:kkt} are presented in the supplementary material.
\end{proof}


\section{Theoretical Analysis}
\label{section:theory}

For the sake of analysis, we adopt the definition of the $T$-period regret and Bayesian regret from \cite{DBLP:journals/mor/RussoR14} to our setting. In particular, we define a policy $\pi$ as a mapping from the history $H_t = (x_1, w_1, P_1, ..., x_{t-1}, w_{t-1}, P_{t-1})$ to $(x_t,w_t, P_t)$ where $P_i \in \mathcal{P}_{n,\rho} \times \mathcal{X}, \forall i$. 

\begin{defn}[$T$-period regret] The $T$-period regret of a policy $\pi$ is defined by
\begin{align*}
    Regret(T, \pi, f) = \sum_{t=1}^T \mathbb{E} \left[ g(f, x^*, P^*) - g(f, x_t, P_t) | f \right],
\end{align*}
where $P^*(.|x) = \argmin_{P \in \mathcal{P}_{n,\rho}}\mathbb{E}_{P(w)}[f(x,w)], \forall x \in \mathcal{X}$, $x^* \in \argmax_{x \in \mathcal{X}} \min_{P \in \mathcal{P}_{n,\rho}} \mathbb{E}_{P(w)}[f(x,w)] $, and for all $T \in \mathbb{N}$. 
\end{defn}
\begin{defn}[$T$-period Bayesian regret]
The $T$-period Bayesian regret of a policy $\pi$ is the expectation of the regret with respect to the prior over $f$, 
\begin{align}
    BayesRegret(T,\pi) &= \mathbb{E} [Regret(T, \pi, f)].
\end{align}
\end{defn}
For simplicity, we focus our analysis on the case where $\mathcal{X}$ is finite and $\mathcal{P}_{n,\rho}$ is a finite subset of the $\chi^2$ ball of radius $\rho$. Similar to \cite{DBLP:conf/icml/SrinivasKKS10}, the results can be extended to infinite sets $\mathcal{X}$ and the entire $\chi^2$ ball using discretization trick of \cite{DBLP:conf/icml/SrinivasKKS10} as long as a smoothness condition (i.e., the partial derivatives of $f$ are bounded with high probability) is satisfied (see Theorem 2 in \cite{DBLP:conf/icml/SrinivasKKS10}).  

\begin{thm} 
\label{thm:bayesregret}
Assume $\mathcal{X}$ is a finite subset of $\mathbb{R}^d$, and $\mathcal{P}_{n,\rho}$ is a finite subset of the $\chi^2$ ball of radius $\rho$. Let $\pi^{DRBQO}$ be the DRBQO policy presented in Algorithm \ref{pseudo_code:drbqo}, $\gamma_T$ be the maximum information gain defined in \cite{DBLP:conf/icml/SrinivasKKS10}, then for all $T \in \mathbb{N}$, 
\begin{align*}
    &BayesRegret(T, \pi^{DRBQO}) \leq 1  \\
    &+ \frac{(\sqrt{2\log \frac{(1+ T^2) |\mathcal{X}| |\mathcal{P}_{n,\rho}|}{ \sqrt{2\pi}} } +B) \sqrt{2\pi}}{|\mathcal{X}| |\mathcal{P}_{n,\rho}|} + \frac{2 \gamma_T \sqrt{(1 + 2\rho)n}}{1 + \sigma^{-2}} \\ 
    &+ 2\sqrt{ T \gamma_T (1 + \sigma^{-2})^{-1} \log \frac{(1+ T^2) |\mathcal{X}| |\mathcal{P}_{n,\rho}|}{ \sqrt{2\pi}}  }.
\end{align*}

\end{thm}
Note that $\gamma_T$ can be bounded for three common kernels: linear, SE and Mat\'ern kernels in \cite{DBLP:conf/icml/SrinivasKKS10}. Using these bounds, Theorem \ref{thm:bayesregret} suggests that DRBQO has sublinear Bayesian regret for common kernels such as linear, SE and Mat\'ern kernels. 
\begin{proof}
We leverage two proof techniques from \cite{DBLP:journals/mor/RussoR14} to derive this bound including posterior sampling regret decomposition and the connection between posterior sampling and UCB. However, an extension from the Bayesian regret bound to our case is non-trivial. The main difficulty is that the $\rho$-robust quadrature distributions $\argmin_{P \in \mathcal{P}_{n,\rho}} \mathbb{E}_{P(w)}[f(x,w)]$ are random variables and the resulting quadrature $\min_{P \in \mathcal{P}_{n,\rho}} \mathbb{E}_{P(w)}[f(x,w)]$ does not follow a GP. We overcome this difficulty by decomposing the range $\mathbb{R}$ of $f(x,w)$ into a set of carefully designed disjoint subsets, using several concentration inequalities for Gaussian distributions, and leveraging the mild assumptions of $f$ from the problem setup. The details are presented in the {supplementary material}. 
\end{proof}

\section{Experiment}
\label{section:experiment}
In this section, we empirically validate the performance of DRBQO by comparing against several baselines in synthetic and $n$-fold cross-validation hyperparameter tuning experiments.  

We focus on the BQO baselines that directly substitute the inferred posterior mean $\mu_t(x, \hat{P}_n)$ (in Equation (\ref{eq:quadrature_mu})) and variance $\sigma_t^2(x, \hat{P}_n)$ (in Equation (\ref{eq:quadrature_sigma})) of  $g(f,x, \hat{P}_n)$ into any standard acquisition (e.g., EI and GP-UCB) to achieve small regret in maximizing $g(f,x, \hat{P}_n)$. More advanced BQO baseline methods, e.g.,  \citep{Toscano_IntegralBO_18} are expected to perform poorly in the distributional uncertainty setting because they are not set out to account for the robust solutions. There is a distinction between sampled points and report points by each baseline algorithm. A sampled point is a suggested point regarding where to sample next while a report point is chosen from all the sampled points (up to any iteration) based on the objective function that an algorithm aims at optimizing. In standard noiseless Bayesian optimization, sampled points and report points are identical. However, this is not necessarily the case in BQO where the objective function has expectation form and is not directly queried.  In particular, we consider the following baselines:   


\begin{itemize}
    \item MTBO: Multi-task Bayesian optimization  \citep{DBLP:conf/nips/SwerskySA13} is a typical BQO algorithm in which  the inferred posterior mean and variance are plugged into the standard EI acquisition to select $x_t$. In addition, each $w_t$ in this case represents a task and MTBO uses multi-task kernels to model the task covariance. Conditioned on $x_t$,  $w_t$ is selected such that the corresponding task yields the highest EI. We include MTBO only in the cross-validation hyperparameter tuning experiments.
    \item BQO-EI: This algorithm is similar to MTBO except for two distinctions. First,  $w_t$ is selected such that it yields the highest posterior variance on $f$, similar to our algorithm (see line \ref{alg:line:w} in Algorithm
    \ref{pseudo_code:drbqo}). Second, this uses kernels defined on the Cartesian product space $\mathcal{X} \times \Omega$ instead of the multi-task kernels as in MTBO. In addition, the report point at time $t$ is $\argmax_{ x \in x_{1:t}}  \mathbb{E}_{\hat{P}_n(w)} [\mu_t(x,w)]$. 

    \item Maximin-BQO-EI: This method is the same as BQO-EI except that the report point is $\argmax_{ x \in x_{1:t}} \min_{P \in \mathcal{P}_{n,\rho}}  \mathbb{E}_{P(w)} [\mu_t(x,w)]$. 
    
    \item BQO-TS: This method is a non-robust version of our proposed DRBQO. The only distinctions between BQO-TS and DRBQO are in the way $x_t$ is selected (line \ref{alg:line:dro} of Algorithm \ref{pseudo_code:drbqo}) and the way a report point is chosen. In BQO-TS, $x_t$ is selected with respect to the empirical distribution as follows: $x_t \in \argmax_{x \in \mathcal{X}} \mathbb{E}_{\hat{P}_n(w)} [\tilde{f}_t(x,w)] $, and the report point at time $t$ is chosen as $\argmax_{ x \in x_{1:t}}  \mathbb{E}_{\hat{P}_n(w)} [\mu_t(x,w)]$. 
    
    
    \item Maximin-BQO-TS: This is the same as BQO-TS except that the final report point is $\argmax_{ x \in x_{1:t}} \min_{P \in \mathcal{P}_{n,\rho}}  \mathbb{E}_{P(w)} [\mu_t(x,w)]$. 
    
    \item Emp-DRBQO: This is the same as DRBQO except that the report point is chosen as $\argmax_{ x \in x_{1:t}}  \mathbb{E}_{\hat{P}_n(w)} [\mu_t(x,w)]$.
\end{itemize}


\begin{figure}[t]
    \centering
    \includegraphics[scale=0.43]{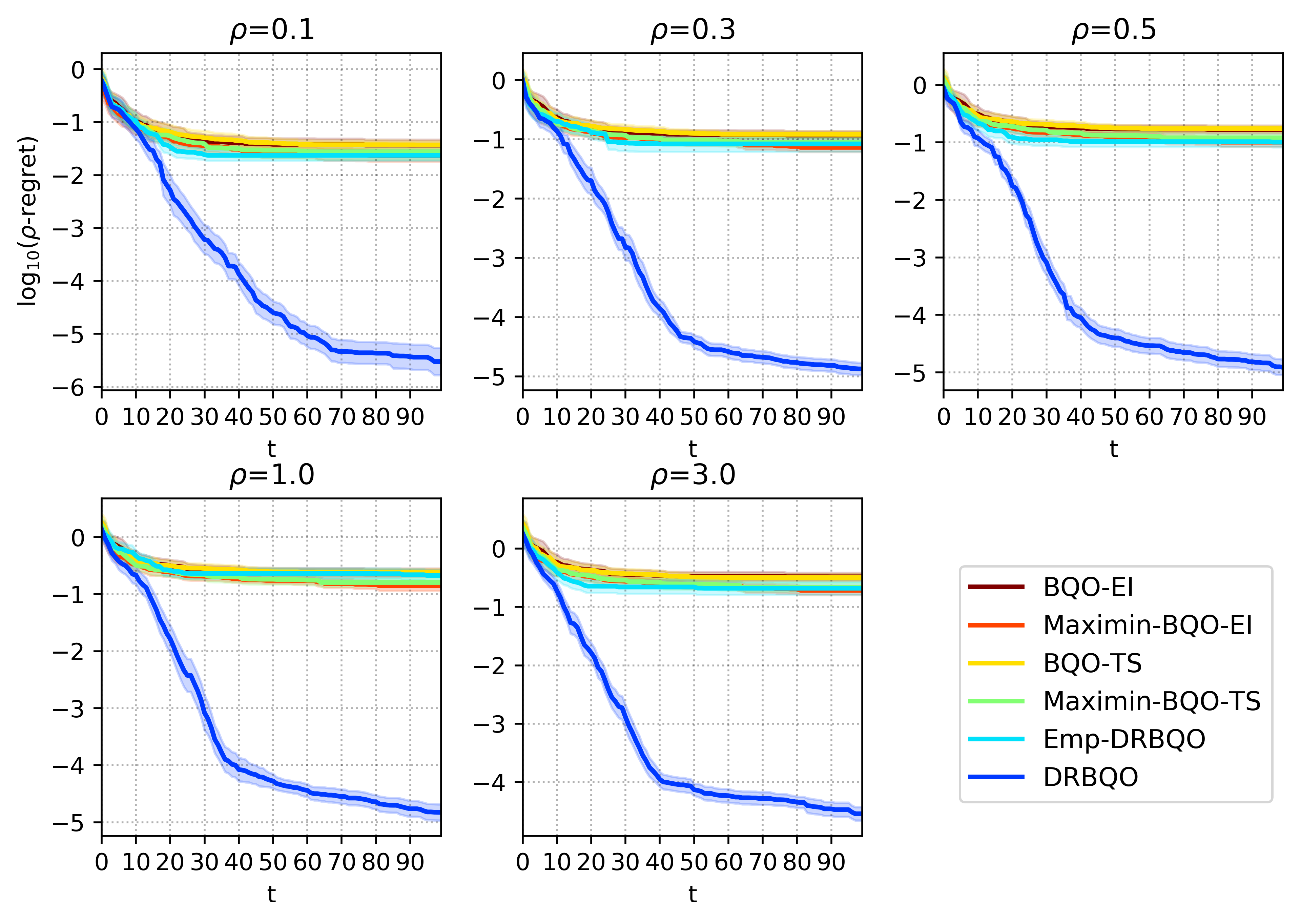}

    \caption{The best so-far $\rho$-regret values (plotted on the $\log_{10}$ scale) of the baseline BQO methods and our proposed method DRBQO for the synthetic function in Section \ref{section:experiment}. DRBQO significantly outperforms the baselines with respect to the $\rho$-regret in this experiment. {The larger the value of} $\rho$ (i.e., {the more conservative against the adversarial distributional perturbation), the higher is the} $\rho$-regret of the non-robust baselines.}
    \label{fig:plot_rho_regret_logistic}
\end{figure}

\textbf{Synthetic function}. 
The distributional uncertainty problem is more pronounced when $f(x,w)$ is more significantly distinct across different values of $w \in S_n$, i.e., $f(x,w)$ experiences high variance along the dimension of $w$. Inspired by the logistic regression and the experimental evaluation from the original variance-based regularization work \citep{NamkoongD16}, we use a logistic form for synthetic function: $f(x,w) = -\log(1 + \exp(x^T w))$, where $x,w \in \mathbb{R}^d$. The true distribution $P_0$ is the standard Gaussian $\mathcal{N}(0,I)$. In this example, we use $d=2$ for better visualization. We sample $n=10$ values of $w$ from $\mathcal{N}(w; 0,I)$ and fix this set for the empirical distribution $\hat{P}_n(w) = \frac{1}{n} \sum_{i=1}^n \delta(w-w_i)$. The true expected function $\mathbb{E}_{P_0(w)}[f(x,w)]$ and the empirical (Monte Carlo) estimate function $\mathbb{E}_{\hat{P}_n(w)}[f(x,w)]$ are illustrated in Figure \ref{fig:figure1} (a) and (b), respectively. In this illustration, the Monte Carlo estimate function catastrophically shifts the true optimum to a spurious point due to the limited data in estimating $P_0$.

\begin{figure}
    \centering
    \includegraphics[scale=0.52]{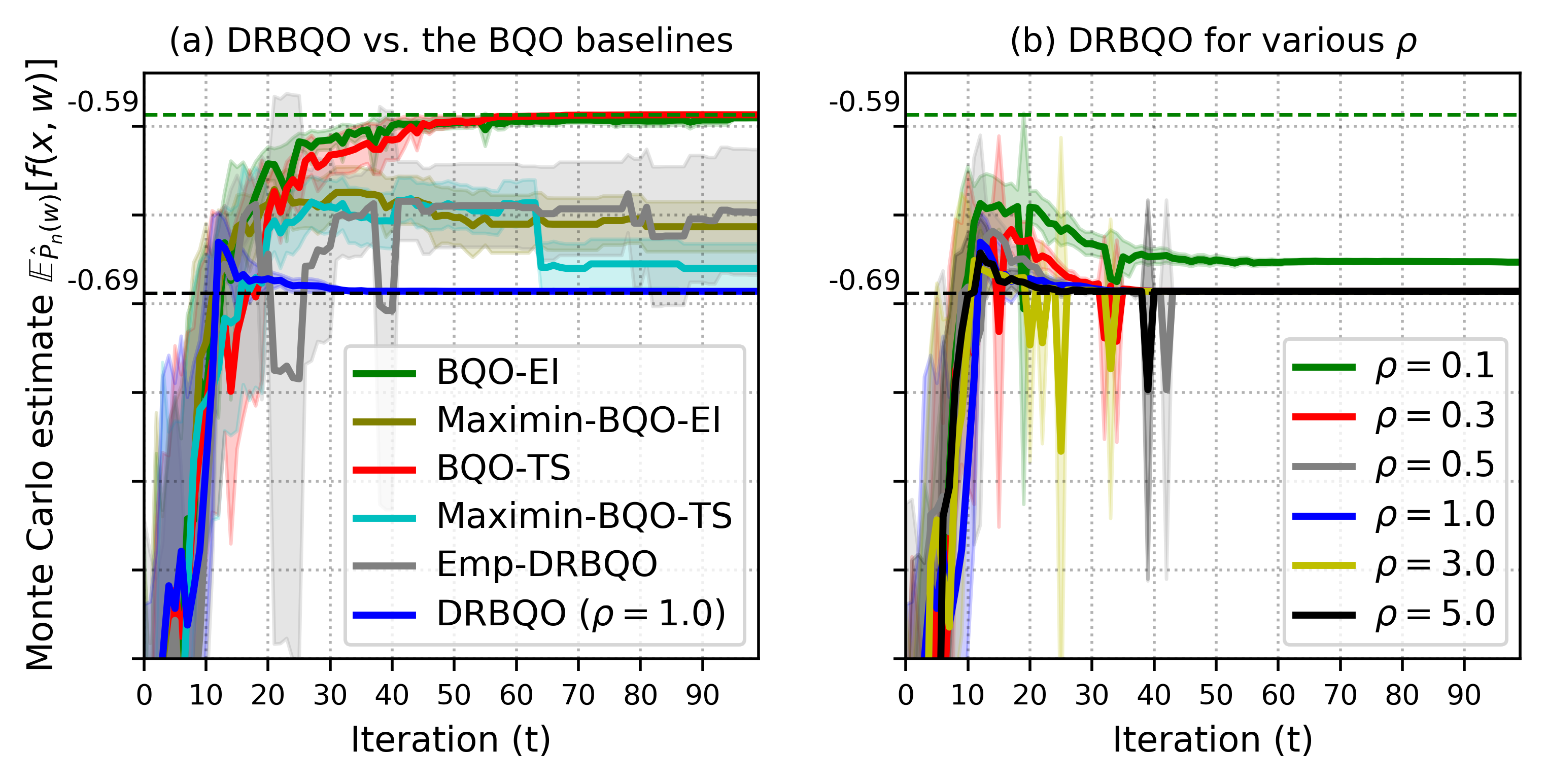}
    \caption{The empirical expected function $\mathbb{E}_{\hat{P}_n(w)}[f(x,w)]$ evaluated at each point $x$ reported at time $t$ by DRBQO and the standard BQO baselines (a), and by DRBQO for various values of $\rho$ (b). In this example, $\mathbb{E}_{\hat{P}_n(w)}[f(x,w)]$ has a maximum value of $-0.59$ while it has a value of $-0.69$ evaluated at the optimum of the true expected function $\mathbb{E}_{P_0(w)}[f(x,w)]$. 
    The BQO baselines achieve higher values of the empirical expected function than DRBQO but DRBQO converges to the distributionally robust solutions.}
    \label{fig:mc_values}
\end{figure}

We initialize the comparative algorithms by selecting $12$ uniformly random inputs $(x,w) \in \mathcal{X}\times S_n$, and {we} keep these initial points the same for all the algorithms. We use the squared exponential kernel $k_{SE}$ defined on the Cartesian product space of $x$ and $w$. We normalize the input and output values to the unit cube, and resort to marginal maximum likelihood to learn the GP hyperparameters \citep{RasmussenW06} every time we acquire a new observation. The time horizon for all the algorithms is $T=100$. We report the results using two evaluation metrics: the $\rho$-regret as defined in Equation (\ref{eq:rho_regret}) and the value of the empirical expected function $\mathbb{E}_{\hat{P}_n(w)}[f(x,w)]$ evaluated at point $x$ reported by an algorithm at time $t$. The former metric quantifies how close a certain point is to the distributionally robust solution while the latter measures the performance of each algorithm from a perspective of the empirical distribution. We repeat the experiment $30$ times and report the average mean and the $96\%$ confidence interval for each evaluation metric. 

The first results are presented in Figure \ref{fig:plot_rho_regret_logistic}. We report over a range of $\rho$ values $\{0.1, 0.3, 0.5, 1.0, 3.0\}$ capturing the degree of conservativeness against the distributional uncertainty. Note that if $\rho > \frac{n-1}{2} = 4.5$, it represents the most conservative case as the $\chi^2$ ball covers the entire $n$-dimensional simplex. We observe from Figure \ref{fig:plot_rho_regret_logistic} that DRBQO significantly outperforms the baselines in this experiment. Also notice that when we increase the conservativeness requirement (i.e., increasing the values of $\rho$), the standard BQO baselines have higher $\rho$-regret. This is because the standard BQO baselines are rigid and do not allow for any conservativeness in the optimization. Therefore, these algorithms converge to the optimum of the spurious Monte Carlo estimate function. 




We highlight the comparative algorithms in the second metric in Figure \ref{fig:mc_values} where we report the value of the empirical expected function $\mathbb{E}_{\hat{P}_n(w)}[f(x,w)]$ at each point $x$ reported by each algorithm at time $t$. Since the BQO baselines are set out to maximize the Monte Carlo estimate function, they achieve higher values in this metric than DRBQO. However, the non-robust solutions returned by the BQO baselines are sub-optimal with respect to the $\rho$-regret in this case, as seen from the corresponding results in Figure \ref{fig:plot_rho_regret_logistic}.

In addition, we evaluate the effectiveness of the selection of $w$ at line \ref{alg:line:w} in Algorithm \ref{pseudo_code:drbqo}. Currently, $w_t$ is selected such that it yields the highest posterior mean given $x_t$. This is to improve exploration in $f$. We compare this selection strategy with the random strategy in which $w_t$ is uniformly selected from $S_n$ regardless of $x_t$. The result is reported in Figure \ref{fig:w_selection}. In this figure, the post-fix RandW denotes the random selection of $w_t$. We observe that random selection of $w_t$ can hurt the convergence of both the standard BQO baselines and DRBQO. Furthermore, the selection of $w_t$ for the highest posterior variance (line \ref{alg:line:w} of Algorithm \ref{pseudo_code:drbqo}) in DRBQO is also meaningful in proving Theorem \ref{thm:bayesregret}. More empirical evaluations in other synthetic functions are presented in the {supplementary material}. 

\begin{figure}
    \centering
    \includegraphics[scale=0.52]{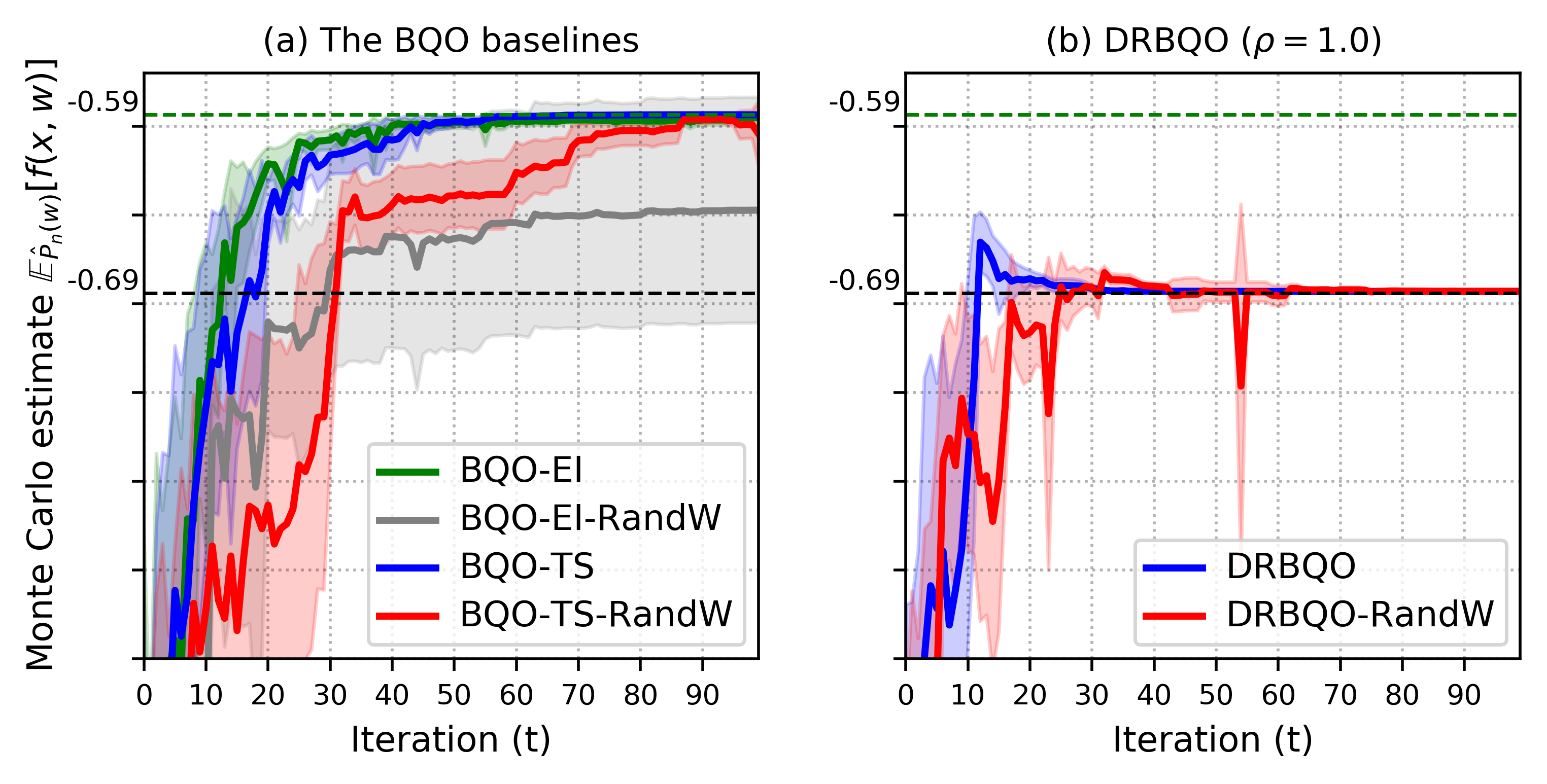}
    \caption{{The effect of different methods of selecting} $w_t$ on the performance of the BQO baselines (a) and DRBQO (b). We observe that random selection of $w_t$ can either slow down or prevent the convergence of both the standard BQO baselines and DRBQO in this experiment.}
    \label{fig:w_selection}
\end{figure}

\textbf{Cross-validation hyperparameter tuning}. A typical real-world problem that possesses the quadrature structure of Equation ($\ref{eq:stochastic_opt}$) is $n$-fold cross-validation hyperparameter tuning. The $n$-fold cross-validation performance can be thought of as a Monte Carlo approximate of the true model performance. Given a fixed learning algorithm associated with a set of hyperparameter $x$, let $f(x,w)$ be an approximate model performance trained on $\mathcal{D} \backslash w$ and evaluated on the validation set $w$ where $\mathcal{D}$ denotes the training data set, $w$ denotes a subset of training points sampled from $\mathcal{D}$, and $\mathcal{D} \backslash w$ denotes everything in $\mathcal{D}$ but not in $w$. Increasing the number of folds reduces the variance in the model performance estimate, but it is expensive to evaluate the cross-validation performance for a large value of $n$. Therefore, a class of Bayesian quadrature optimization methods is beneficial in this case in which we actively select both the algorithm's hyperparameters $x_t$ and a fold $w_t$ to evaluate without the need of training the model in all $n$ folds \citep{DBLP:conf/nips/SwerskySA13}. 

However, the standard BQO methods assume the empirical distribution for each fold and are set out to maximize the average $n$-fold values.  In practice, the average $n$-fold value can be a spurious measure for model performance when there is sufficient discrepancy of the model performance across different folds. This scenario fits well into our distributional uncertainty problem in Equation (\ref{eq:stochastic_opt}) where the fold distribution $P_0(w)$ is unknown in practice. 
In addition, we use a one-hot $n$-dimensional vector to represent each of the $n$ folds. This offers two main advantages: (i) it allows us to leverage the standard kernel such as $k_{SE}$ on the product space $\mathcal{X} \times \Omega$; (ii) it is able to model different covariance between different pairs of folds. For example, the covariance between fold 1 and fold 3 is not necessary the same as that between fold 8 and fold 10 though the fold indicator difference are the same ($2 = 10 - 8 = 3 -1$ in this example). 

We evaluate this experiment on two common machine learning models using the MNIST dataset \citep{LeCun98}: ElasticNet and Convolutional Neural Network (CNN). For ElasticNet, we tune the $l_1$ and $l_2$ regularization hyperparamters, and use the SGDClassificer implementation from the scikit-learn package \citep{pedregosa2012scikitlearn}. For CNN, we use the standard architecture with 2 convolutional layers.
In CNN, we optimize over three following hyperparamters: the learning rate $l$ and the dropout rates in the first and second pooling layers. We used Adam optimizer \citep{DBLP:journals/corr/KingmaB14} in $20$ epochs with the batch size of $128$. 

\begin{table}
\caption{Classification error (\%) of ElasticNet and CNN on the MNIST test set tuned by different algorithms. Each bold number in the DRBQO group denotes the classification error that is smaller than any corresponding number in the baseline group. } \label{table:cls_err}
\begin{center}
\begin{tabular}{lccc}
\textbf{Methods}  &\textbf{ElasticNet} &\textbf{CNN} \\
\hline
MTBO       & $8.576 \pm 0.080$      & $1.712 \pm 0.263$ \\
BQO-EI      & $9.166 \pm 0.433$      & $1.634 \pm 0.157$\\
BQO-TS      & $8.625 \pm 0.116$      & $1.820 \pm 0.227$\\ 
\hline
DRBQO($\rho=0.1$)   & $\pmb{8.450} \pm 0.022$   &  $1.968 \pm 0.310$ \\
DRBQO($\rho=0.3$)   &  $\pmb{8.505} \pm 0.082$    & $\pmb{1.495} \pm 0.106$ \\ 
DRBQO($\rho=0.5$)   & $\pmb{8.515} \pm 0.075$   & $1.869 \pm 0.232$\\ 
DRBQO($\rho=1$)   &  $\pmb{8.526} \pm 0.065$  & $\pmb{1.444} \pm 0.071$ \\ 
DRBQO($\rho=3$)   & $\pmb{8.387} \pm 0.013$ &  $\pmb{1.374} \pm 0.066$ \\ 
DRBQO($\rho=5$)   &  $\pmb{8.380} \pm 0.022$ &  $\pmb{1.321} \pm 0.061$ \\ 
\end{tabular}
\end{center}
\end{table}

In addition to the previous baselines in the synthetic experiment, we also consider the multi-task Bayesian optimization (MTBO) \citep{DBLP:conf/nips/SwerskySA13} baseline for this application. MTBO is a standard method for cross-validation hyperparameter tuning. 

In this experiment, we also use $k_{SE}$ kernel defined on the Cartesian product space $\mathcal{X} \times \Omega$ of $x$ and $w$ for all the methods except for MTBO which uses task kernel on the domain of $w$. We initialize $6$ (respectively $9$) initial points and keep these initial points the same for all the algorithms in ElasticNet (respectively CNN). Each of the algorithms are run for $T=60$ (respectively $T=90$) iterations in ElasticNet (respectively CNN). We repeat the experiment $20$ times and report the average and standard deviation values of an evaluation metric. We split the training data into $n=10$ folds and keep these folds the same for all algorithms. We compare DRBQO against the baselines via a practical metric: the classification error in the test set evaluated at the final set of hyperparameters reported by each algorithm at the final step $T$. This metric is a simple but practical measure of the robustness of the hyperparameters over the unknown data distribution $P_0$. The result is reported in Table \ref{table:cls_err}. We observe that DRBQO outperforms the baselines for most of the considered values of $\rho$, especially for large values of $\rho$ (i.e., $\rho \in \{1, 3, 5\}$ in this case). More results for the case of Support Vector Machine (SVM) are presented in the supplementary material.

\section{Discussion}
\label{section:discussion}
In this work, we have proposed a posterior sampling based algorithm, namely DRBQO, that efficiently seeks for the robust solutions under the distributional uncertainty in
 Bayesian quadrature optimization. Compared to the standard BQO algorithms, DRBQO provides a flexibility to control the conservativeness against distributional perturbation. We have demonstrated the empirical effectiveness and characterized the theoretical convergence of DRBQO in sublinear Bayesian regret. 



\newpage
\subsubsection*{Acknowledgements}
This research was partially funded by the Australian Government through the Australian Research Council (ARC). Prof Venkatesh is the recipient of an ARC Australian Laureate Fellowship (FL170100006).

\bibliographystyle{plainnat}
\bibliography{refs}

\newpage
\section*{Appendix A}
In this appendix, we provide a detailed proof of Theorem 2 in the main text about a sublinear bound on the Bayesian regret of the DRBQO algorithm. For simplicity, we focus on the case where the decision space $\mathcal{X}$ and the distributional uncertainty set $\mathcal{P}_{n,\rho}$ are finite. The results can be extended to infinite sets using the discretization trick as in \cite{DBLP:conf/icml/SrinivasKKS10}.

\textbf{Notations and conventions}. Unless explicitly specified otherwise, we denote  a conditional distribution $P(. | x) \in \mathcal{P}_{n,\rho}, \forall x \in \mathcal{X}$ by $P$, i.e., $P \in  \mathcal{P}_{n,\rho} \times \mathcal{X}$. Recall the definition of the quadrature functional in the main text as
\begin{align*}
    g(f,x,P) = \int P(w|x) f(x,w) dw,
\end{align*}
for any $x \in \mathcal{X}$ and $P \in \mathcal{P}_{n,\rho} \times \mathcal{X}$.
Let $x^* \in \argmax_{x \in \mathcal{X}} \min_{P \in \mathcal{P}_{n,\rho}} \mathbb{E}_{P(w)}[f(x,w)] $, and $P^*(.|x) = \argmin_{P \in \mathcal{P}_{n,\rho}}\mathbb{E}_{P(w)}[f(x,w)], \forall x \in \mathcal{X}$. Since $f$ is a stochastic process (a GP in our case), $x^*$ and $P^*$ are also random variables. The DRBQO algorithm $\pi^{DRBQO}$ maps at a time step $t$ the history $H_t = (x_1, w_1, P_1, ..., x_{t-1},w_{t-1}, P_{t-1})$ to a new decision $(x_t,w_t) \in \mathcal{X} \times S_n$ and conditional distribution $P_t \in \mathcal{P}_{n,\rho} \times \mathcal{X}$ as presented in line 2-3 of Algorithm 1 in the main text. The practical implementation of Algorithm 1 samples $(x_t, P_t)$ as follows: 
$x_t \in \argmax_{x \in \mathcal{X}} \min_{P \in \mathcal{P}_{n,\rho}} \mathbb{E}_{P(w)}[\tilde{f}_t(x,w)] $, and $P_t(.|x) = \argmin_{P \in \mathcal{P}_{n,\rho}}\mathbb{E}_{P(w)}[\tilde{f}_t(x,w)], \forall x \in \mathcal{X}$ where $\tilde{f}_t$ is a function sample of $f$ at time $t$ {from its posterior GP}.

\begin{lem}
\label{lem:ps_decompose}
For any sequence of deterministic functions $\{ U_t: \mathcal{X} \times \mathcal{P}_{n,\rho} \times \mathcal{X} \rightarrow \mathbb{R} | t \in \mathbb{N} \}$, 
\begin{align*}
    &BayesRegret(T, \pi^{DRBQO}) \nonumber\\
    &= \mathbb{E} \sum_{t=1}^T \left[ U_t(x_t, P_t) - g(f, x_t, P_t)  \right] \nonumber \\ 
    &+\mathbb{E} \sum_{t=1}^T \left[ g(f, x^*, P^*) - U_t(x^*, P^*)    \right],
\end{align*}
for all $T \in \mathbb{N}$. 
\end{lem}

\begin{proof}
Given $H_t$, $\pi^{DRBQO}$ samples $(x_t, P_t)$ according to the probability they are optimal, i.e., $(x_t, P_t) \sim Pr( x^*, P^* | H_t)$. Thus, conditioned on $H_t$, $(x^*, P^*)$ and $(x_t, P_t)$ are identically distributed. As a result, given a deterministic function $U_t$, we have $\mathbb{E}[U_t(x^*, P^*)] = \mathbb{E}[ U_t(x_t, P_t)]$. Therefore, 
\begin{align*}
    &\mathbb{E} \left[ g(f, x^*, P^*) - g(f, x_t, P_t) \right] \\
    &=  \mathbb{E} \left[ 
    \mathbb{E} \left[ g(f, x^*, P^*) - g(f, x_t, P_t) \right] | H_t
    \right] \\
    &= \mathbb{E} \left[ 
    \mathbb{E} \left[  U_t(x_t, P_t) - g(x_t, P_t) \right] | H_t
    \right]  \\ 
    &+ \mathbb{E} \left[ 
    \mathbb{E} \left[   g(f, x^*, P^*) - U_t(x^*, P^*)  \right] | H_t
    \right] \\
    &= \mathbb{E} \left[  U_t(x_t, P_t) - g(f, x_t, P_t) \right] \\
    &+ \mathbb{E} \left[   g(f, x^*, P^*) - U_t(x^*, P^*)  \right].
\end{align*}
\end{proof}

\begin{lem}
\label{lem:lem2}
Let $X \sim \mathcal{N}(\mu, \sigma^2)$.  
\begin{enumerate}
    \item For all $\beta \geq 0$, we have 
    \begin{align*}
    Pr\{ |X - \mu| > \beta^{1/2} \sigma \} \leq e^{-\beta/2}. 
\end{align*}

\item If $\mu \leq 0$, then 
\begin{align*}
    \mathbb{E}[ \max\{X,0\} ] = \frac{\sigma}{\sqrt{2\pi}} e^{ \frac{-\mu^2}{2 \sigma^2} }.
\end{align*}

\item For all $a \leq b$, we have 
\begin{align*}
    \mathbb{E}[ X | a < X < b] = \mu - \sigma^2 \frac{p(a) - p(b)}{\phi(a) - \phi(b)}, 
\end{align*}
where $p(x)$ and $\phi(x)$ denote the density function and cumulative distribution function of $X$, respectively. 
\end{enumerate}
\end{lem}
\begin{proof}
The results are simple properties of normal distributions. 
\end{proof}

\begin{lem}
\label{lem:cov_inequality}
Given $H_{t}, \forall t \in \mathbb{N}$, let $\sigma_{t}^2(x,w) := C_{t}(x,w; x,w)$  be the variance of $f(x,w)$. Then, for all $P$, all $x$ and for $w^* = \argmax_{w \in S_n} \sigma^2_{t}(x,w)$, we have
\begin{align*}
    \sigma^2_{t}(x, P) = Var[g(f, x, P) | H_t] \leq  \sigma_{t}^2(x, w^*). 
\end{align*}
\end{lem}
\begin{proof}
It follows from a simple property of posterior covariance that 
\begin{align*}
    &\sigma^2_{t}(x, P) = \sum_{w, w'} P(w|x) P(w'|x) C_{t}(x, w; x, w') \nonumber \\
    & \leq \sum_{w, w'} P(w|x) P(w'|x) C_{t}(x, w; x, w) \nonumber \\ 
    & \leq \sum_{w, w'} P(w|x) P_t(w'|x) \sigma_{t-1}^2(x, w^*) \nonumber\\
    & = \sigma_{t}^2(x, w^*). 
\end{align*}
\end{proof}

\begin{lem}
\label{lemma:second_term}
If $U_t(x,P) = \mu_{t-1}(x,P) + \sqrt{\beta_t} \sigma_{t-1}(x,P) $ where 
\begin{align*}
    &\mu_{t-1}(x,P) := \int P(w|x) \mu_{t-1} (x,w) dw, \\ 
    &\sigma^2_{t-1}(x,P) := \nonumber \\
    &\int \int C_{t-1}(x,w; x,w') P(w|x) P(w'|x) dw dw', 
\end{align*}
and $\beta_t = 2 \log \frac{(t^2 + 1)|\mathcal{X}| |\mathcal{P}_{n,\rho}|}{\sqrt{2\pi}}$, then 
\begin{align*}
    \mathbb{E} \sum_{t=1}^T [g(f, x^*, P^*) - U_t(x^*, P^*)] \leq 1,
\end{align*}
for all $T \in \mathbb{N}$.
\end{lem}

\begin{proof}
The trick is to concentrate on the non-negative terms of the expectation. These non-negative terms can be bounded due to the specific choice of upper confidence bound $U_t$. 


Note that for any deterministic conditional distribution $P \in \mathcal{P}_{n,\rho} \times \mathcal{X}$, we have $g(f,x,P) \sim \mathcal{N}(\mu_{t-1}(x,P), \sigma_{t-1}^2(x,P))$, i.e., $g(f, x,P) - U_t(x,P) \sim \mathcal{N}(-\sqrt{\beta_t} \sigma_{t-1}(x,P),  \sigma_{t-1}^2(x,P) )$. It thus follows from Lemma \ref{lem:lem2}.2 that: 
\begin{align*}
    &\mathbb{E}[ \max\{  g(f,x,P) - U_t(x,P), 0 \} | H_t] \\ &= \frac{\sigma_{t-1}(x,P)}{\sqrt{2\pi}} \exp(\frac{-\beta_t}{2}) \\ 
    &= \frac{\sigma_{t-1}(x,P)}{(t^2 +1) |\mathcal{X}| |\mathcal{P}_{n,\rho}|} \leq \frac{1}{(t^2 +1) |\mathcal{X}| |\mathcal{P}_{n,\rho}|}.
\end{align*}

The final inequality above follows from Lemma \ref{lem:cov_inequality} and from the assumption that $\sigma_0(x,w) \leq 1, \forall x,w$, i.e.,  
\begin{align*}
    \sigma_{t-1}(x,P) \leq \sigma_{t-1}(x,w^*) \leq \sigma_0 (x,w^*) \leq 1, 
\end{align*}
where $w^* = \argmax_{w} C_{t-1}(x,w; x, w)$. 

Therefore, we have 
\begin{align*}
    &\mathbb{E} \sum_{t=1}^T [g(f, x^*, P^*) - U_t(x^*, P^*)] \\
    &\leq  \mathbb{E} \sum_{t=1}^T \mathbb{E} [ \max \{g(f, x^*, P^*) - U_t(x^*, P^*), 0 \} | H_t] \\ 
    &\leq \mathbb{E} \sum_{t=1}^T \sum_{x \in \mathcal{X}} \sum_{P \in \mathcal{P}_{n,\rho}}  \mathbb{E} [ \max \{g(f, x, P) - U_t(x, P), 0 \}] \\ 
    &\leq \sum_{t=1}^{\infty} \sum_{x \in \mathcal{X}} \sum_{P \in \mathcal{P}_{n,\rho}} \frac{1}{(t^2 +1) |\mathcal{X}| |\mathcal{P}_{n,\rho}|} = 1. 
\end{align*}
\end{proof}

\begin{lem} 
\label{lem:first_term}
{Given} the definition of the maximum information gain $\gamma_T$ as in \cite{DBLP:conf/icml/SrinivasKKS10}, we have
\begin{align*}
    &\mathbb{E} \sum_{t=1}^T \left[ U_t(x_t, P_t) - g(f,x_t, P_t)  \right]  \\ 
    &\leq \frac{(\sqrt{\beta_T}+B) \sqrt{2\pi}}{|\mathcal{X}| |\mathcal{P}_{n,\rho}|} + 2 \gamma_T \sqrt{(1 + 2\rho)n} (1 + \sigma^{-2})^{-1} \\ 
    &+ 2\sqrt{ T \gamma_T (1 + \sigma^{-2})^{-1} \log \frac{(1+ T^2) |\mathcal{X}| |\mathcal{P}_{n,\rho}|}{ \sqrt{2\pi}}  } ,
\end{align*}
for all $T \in \mathbb{N}$. 
\end{lem}

\begin{proof}
Now we bound the first term 
\begin{align*}
    L &:= \mathbb{E} \sum_{t=1}^T \left[ U_t(x_t, P_t) - g(f, x_t, P_t)  \right] \\
    &= \mathbb{E} \sum_{t=1}^T \mathbb{E} [J(x_t, H_t) | x_t, H_t], 
\end{align*}
where
\begin{align*}
    J(x_t, H_t) = \mathbb{E}[  U_t(x_t, P_t) - g(f, x_t, P_t) | H_t, x_t ]. 
\end{align*}

While the second term of the Bayesian regret of DRBQO can be bounded as in Lemma \ref{lemma:second_term} by adopting the techniques from \cite{DBLP:journals/mor/RussoR14}, bounding $L$ in DRBQO is non-trivial. This is because $P_t(.|x)$ is also a random process on the simplex given $H_t$. Thus, $g(f, x_t, P_t) | H_t$ does not follow a GP as in the standard Quadrature formulae. In addition, we do not have a closed form of $\mathbb{E}[ g(f, x_t, P_t) | H_t ]$. We overcome this difficulty by decomposing $J$ into several terms that can be bounded more easily and leveraging the mild assumptions of $f$ in the problem setup. 

Given $(H_t, x_t)$, we are interested in bounding $J(x_t, H_t)$. The main idea for bounding this term is that we decompose the range $\mathbb{R}$ of the random variable $f(x_t,w), \forall w$ into three disjoint sets: 
\begin{align*}
    &A_t(w) =\\
    &\bigg \{ f(x_t, w) \bigg| |f(x_t, w) - \mu_{t-1}(x_t, w)| \leq \sqrt{\beta_t} \sigma_{t-1}(x_t, w) \bigg\}, 
\end{align*}

\begin{align*}
    &B_t(w) =\\
    &\bigg \{ f(x_t, w) \bigg|  \mu_{t-1}(x_t, w) - f(x,w) >  \sqrt{\beta_t} \sigma_{t-1}(x_t, w) \bigg\},  
\end{align*}

\begin{align*}
    &C_t(w) =\\
    &\bigg \{ f(x_t, w) \bigg|  \mu_{t-1}(x_t, w) - f(x,w) <  -\sqrt{\beta_t} \sigma_{t-1}(x_t, w) \bigg\}, 
\end{align*}
for all $w \in \Omega$. Note that $A_t(w) \cup B_t(w) \cup C_t(w) = \mathbb{R},  \forall w$.  We also denote $\bar{A}_t(w) = \mathbb{R} \backslash A_t(w) = B_t(w) \cup C_t(w), \forall w$. 

Since $f$ is bounded on $A_t$, there exists $P^*_t$ such that
\begin{align*}
    P^*_t(.|x) = \argmax_{P \in \mathcal{P}_{n,\rho}} \{ U_t(x,P) - g(f, x,P)| f \in A_t \}, 
\end{align*}
for all $x \in \mathcal{X}$. 

Using the equation above, we decompose $J(x_t, H_t)$ as
\begin{align*}
    &J(x_t, H_t) = \mathbb{E}[  U_t(x_t, P_t) - g(f, x_t, P_t)| H_t, x_t] \\
    &=  \mathbb{E}_{f \in A_t}[  U_t(x_t, P_t) - g(f, x_t, P_t)| H_t, x_t] \\
    &+ \mathbb{E}_{f \in \bar{A}_t}[  U_t(x_t, P_t) - g(f, x_t, P_t)| H_t, x_t] \\
    &\leq \mathbb{E}_{f \in A_t}[  U_t(x_t, P^*_t) - g(f, x_t, P^*_t)| H_t, x_t] \\
    &+ \mathbb{E}_{f \in \bar{A}_t}[  U_t(x_t, P_t) - g(f, x_t, P_t)| H_t, x_t] \\
    &= \mathbb{E}[  U_t(x_t, P^*_t) - g(f, x_t, P^*_t)| H_t, x_t] \\
    &+ \mathbb{E}_{f \in \bar{A}_t}[  U_t(x_t, P_t) - g(f, x_t, P_t)| H_t, x_t] \\ 
    &- \mathbb{E}_{f \in \bar{A}_t}[  U_t(x_t, P^*_t) - g(f, x_t, P^*_t)| H_t, x_t] \\ 
    &= J_1 + J_2 + J_3,
\end{align*}
where 
\begin{align*}
    J_1 &= \mathbb{E}[  U_t(x_t, P^*_t) - g(f, x_t, P^*_t)| H_t, x_t], \\
    J_2 &= \mathbb{E}_{f \in \bar{A}_t}[  U_t(x_t, P_t) - g(f, x_t, P_t)| H_t, x_t], \\ 
    J_3 &= \mathbb{E}_{f \in \bar{A}_t}[  g(f, x_t, P^*_t) - U_t(x_t, P^*_t)| H_t, x_t]. 
\end{align*}

It follows from Lemma \ref{lem:cov_inequality} and from the selection of $w_t$ for the highest posterior variance in the DRBQO algorithm (Algorithm 1 in the main text) that for all $P$, we have
\begin{align*}
    &\sigma^2_{t-1}(x_t, P) = \sum_{w, w'} P(w|x) P(w'|x) C_{t-1}(x_t, w; x_t, w') \\
    &\leq \sigma_{t-1}^2(x_t, w_t). 
\end{align*}
Note that given $(H_t, x_t)$, $w_t$ is deterministic.  

For $J_1$, we have
\begin{align*}
    J_1 &= \mathbb{E}[  U_t(x_t, P^*_t) - g(f, x_t, P^*_t)| H_t, x_t] \\ 
    &= U_t(x_t, P^*_t) - \mathbb{E}[ g(f, x_t, P^*_t)| H_t, x_t] \\ 
    &= U_t(x_t, P^*_t) - \mu_{t-1}(x_t, P^*_t) \\ 
    &= \sqrt{\beta_t} \sigma_{t-1}(x_t, P^*_t) \\ 
    &\leq \sqrt{\beta_t} \sigma_{t-1}(x_t, w_t).
\end{align*}

For $J_2$, we have
\begin{align*}
    J_2 &= \mathbb{E}_{f \in \bar{A}_t}[  U_t(x_t, P_t) - g(f, x_t, P_t)| H_t, x_t] \\ 
    &=  \mathbb{E}_{f \in \bar{A}_t} [ \sqrt{\beta_t} \sigma_{t-1}(x_t, P_t)| H_t, x_t] \\
    &+  \mathbb{E}_{f \in B_t} [ \sum_{w} P_t(w) (\mu_{t-1}(x_t, w) - f(x_t, w) | H_t, x_t] \\ 
    &+  \mathbb{E}_{f \in C_t} [ \sum_{w} P_t(w) (\mu_{t-1}(x_t, w) - f(x_t, w)   | H_t, x_t] \\ 
    &\leq \mathbb{E}_{f \in \bar{A}_t}[ \sqrt{\beta_t} \sigma_{t-1}(x_t, w_t) ]\\
    &+  \mathbb{E}_{f \in B_t} [ \sum_{w} P_t(w) (\mu_{t-1}(x_t, w) - f(x_t, w)    | H_t, x_t] \\ 
    &\leq \sqrt{\beta_t} \sigma_{t-1}(x_t, w_t) e^{-\beta_t/2} \\ 
    &+  \mathbb{E}_{f \in B_t} [ \sum_{w} P_t(w) (\mu_{t-1}(x_t, w) - f(x_t, w) | H_t, x_t] \\
    &\leq \sqrt{\beta_t} \sigma_{t-1}(x_t, w_t) e^{-\beta_t/2} \\
    &+ \mathbb{E}_{f \in B_t} \sqrt{\sum_{w}  (\mu_{t-1}(x_t, w) - f(x_t, w))^2 \sum_{w} P_t^2(w)} \\ 
    &\leq \sqrt{\beta_t} \sigma_{t-1}(x_t, w_t) e^{-\beta_t/2} \\
    &+ \mathbb{E}_{f \in B_t} \sqrt{\sum_{w}  (\mu_{t-1}(x_t, w) - f(x_t, w))^2 \frac{1+2\rho}{n}} \\ 
    &\leq \sqrt{\beta_t} \sigma_{t-1}(x_t, w_t) e^{-\beta_t/2} \\
    &+ \sqrt{\frac{1+2\rho}{n}} \mathbb{E}_{f \in B_t} \sum_{w}  (\mu_{t-1}(x_t, w) - f(x_t, w))  \\ 
    &= \sqrt{\beta_t} \sigma_{t-1}(x_t, w_t) e^{-\beta_t/2} \\
    &+ \sqrt{\frac{1+2\rho}{n}} \sum_{w} (\mu_{t-1}(x_t, w) - \mathbb{E}_{f \in B_t}[f(x_t, w)] )  \\
    &= \sqrt{\beta_t} \sigma_{t-1}(x_t, w_t) e^{-\beta_t/2} \\
    &+ \sqrt{\frac{1+2\rho}{n}} \sum_{w} \sigma^2_{t-1}(a_t,w)  \kappa(x_t, w)  \\
    &\leq \sqrt{\beta_t} \sigma_{t-1}(x_t, w_t) e^{-\beta_t/2} + \sqrt{\frac{1+2\rho}{n}} \sum_{w} \sigma^2_{t-1}(a_t,w) \\ 
    &\leq \sqrt{\beta_t} \sigma_{t-1}(x_t, w_t) e^{-\beta_t/2} + \sqrt{n(1+2\rho)} \sigma^2_{t-1}(a_t,w_t),  \\ 
\end{align*}
where 
\begin{align*}
    \kappa(x_t, w) := \frac{p(\mu_{t_1}(x_t,w) - \sqrt{\beta_t}  \sigma_{t-1}(x_t, w) )}{\phi(\mu_{t_1}(x_t,w) - \sqrt{\beta_t}  \sigma_{t-1}(x_t, w) )} \leq 1,
\end{align*}
and $p(.)$ and $\phi(.)$ denote the density function and the cumulative distribution function of the Gaussian distribution $\mathcal{N}(\mu_{t-1}(x_t, w), \sigma^2_{t-1}(x_t, w)), \forall w$. Here, the third inequality follows from the Cauchy-Schwartz inequality; the fourth inequality follows from the bound of the $\chi^2$ ball on the distributions in it; the fifth inequality follows from that fact that $\mu_{t-1}(x,w) - f(x,w) \geq \sqrt{\beta_t} \sigma_{t-1}(x,w) \geq 0$; and the final equation follows from Lemma \ref{lem:lem2}.3. 

For $J_3$, we have 
\begin{align*}
    J_3 &= \mathbb{E}_{f \in \bar{A}_t}[  g(f, x_t, P^*_t) - U_t(x_t, P^*_t)| H_t, x_t] \\ 
    &= \mathbb{E}_{f \in \bar{A}_t}[  g(f, x_t, P^*_t) | H_t, x_t] + \mathbb{E}_{f \in \bar{A}_t} [ -U_t(x_t, P^*_t)] \\ 
    & = \mathbb{E}_{f \in \bar{A}_t}[ -U_t(x_t, P^*_t)] \\ 
    &=  \mathbb{E}_{f \in \bar{A}_t} [- \mu_{t-1}(x_t, P^*_t) - \sqrt{\beta_t} \sigma_{t-1}(x_t, P^*_t) ]\\ 
    &\leq \mathbb{E}_{f \in \bar{A}_t} [- \mu_{t-1}(x_t, P^*_t)] \\ 
    & \leq \mathbb{E}_{f \in \bar{A}_t} [B] \\
    & \leq B e^{-\beta_t/2}. 
\end{align*}

Here, the second equation follows from the property that $\mathbb{E}_{f\in\bar{A}_t}[f(x_t,w)] = 0$ since $f(x_t,w) \sim \mathcal{N}(\mu_{t-1}(x_t, w), \sigma^2_{t-1}(x_t,w)), \forall w$, and $\bar{A}_t(w)$ is a symmetric region in $\mathbb{R}$ with respect to (but not including) the line $x = \mu_{t-1}(x_t,w), \forall w$; the first inequality follows the non-negativity of the posterior variance $\sigma_{t-1}(x_t, P^*_t)$; the second inequality follows from that the posterior mean $\mu_{t-1}(x, w)$ of a GP is in the RKHS associated with kernel $k$ of the GP, thus is bounded above by $B$ by the mild assumption in the problem setup; and the final inequality follows from Lemma \ref{lem:lem2}.1.  

Combining these results, we can finally bound the first term of the Bayesian regret of DRBQO, 

\begin{align*}
    L &= \mathbb{E} \sum_{t=1}^T \mathbb{E} [J(x_t, H_t) | x_t, H_t] \\
    &\leq \mathbb{E} \sum_{t=1}^T \sqrt{\beta_t} \sigma_{t-1}(x_t, w_t) + \mathbb{E} \sum_{t=1}^T B e^{-\beta_t /2 } \\
    &+ \mathbb{E} \sum_{t=1}^T \sqrt{\beta_t} \sigma_{t-1}(x_t, w_t) e^{-\beta_t/2} \\
    &+ \mathbb{E} \sum_{t=1}^T \sqrt{n(1+2\rho)} \sigma^2_{t-1}(a_t,w_t) \\
    &\leq \mathbb{E} \sqrt{T \beta_T} \sqrt{\sum_{t=1}^T \sigma^2_{t-1}(x_t, w_t)} 
    \end{align*}
\begin{align*}
    &+ (B+\sqrt{\beta_T})\sum_{t=1}^{\infty} \frac{\sqrt{2\pi}}{(1+t^2) |\mathcal{X}| |\mathcal{P}_{n,\rho}|} \\ 
    &+  \sqrt{n(1+2\rho)} \mathbb{E} \sum_{t=1}^T \sigma^2_{t-1}(a_t,w_t) \\
    &\leq \sqrt{T \beta_T} \sqrt{2(1+\sigma^{-2})^{-1} \gamma_T } + \frac{(\sqrt{\beta_T}+B) \sqrt{2\pi}}{|\mathcal{X}| |\mathcal{P}_{n,\rho}|} \\ 
    &+ \sqrt{n(1 + 2\rho)} 2(1+\sigma^{-2})^{-1} \gamma_T, 
\end{align*}
where $\gamma_T$ is the maximum information gain defined in \cite{DBLP:conf/icml/SrinivasKKS10}, and we also use the following inequality of the maximum information gain 
\begin{align*}
    \sum_{t=1}^T \sigma_{t-1}^2(x_t, w_t) \leq 2(1 + \sigma^{-2})^{-1} \gamma_T.  
\end{align*}
\end{proof}

Theorem 2 is a direct consequence of Lemma \ref{lem:ps_decompose}, Lemma \ref{lemma:second_term} and Lemma \ref{lem:first_term}. 

\textbf{Upper bounds on the information gain}. For completeness, we include here the upper bounds for the information gains $\gamma_T$ which are derived from \cite{DBLP:conf/icml/SrinivasKKS10}:
\begin{center}
\begin{tabular}{ | c | c |} 
\hline
\textbf{Kernel type} & \textbf{Information gain} $\gamma_T$ \\ 
\hline 
Linear & $\mathcal{O}(d \log T)$  \\ 
\hline
Squared exponential & $\mathcal{O}(\log T)^{d+1})$  \\ 
\hline
Mat\'ern with $\nu > 1$ & $ \mathcal{O}( T^{d(d+1) / (2 \nu + d(d+1))} \log T ) $  \\ 
\hline
\end{tabular}
\end{center}
where $d \in \mathbb{N}$ is the dimension of the search domain. 

\section*{Appendix B.1}
In this appendix, we provides derivation details of Proposition 1 in the main text. 

Consider the constrained optimization problem 
\begin{align}
    \label{eq:rho_quadrate_weights}
    \min_{p \in \mathcal{P}_{n,\rho}} \sum_{i=1}^n p_i l_i. 
\end{align}

This is a convex optimization problem which forms the Lagrangian: 
\begin{align}
    L(p, \lambda, \eta, \zeta) &= p^Tl - \lambda \left(\rho - \frac{1}{2n} \sum_{i=0}^n (n p_i - 1)^2 \right)  \nonumber \\
    &-\eta (1 - \sum_{i=1}^n p_i) - \sum_{i=1}^n \zeta_i p_i, 
\end{align}
where $p \in \mathbb{R}^n, \lambda \geq 0$, $\eta \in \mathbb{R}$, and $\zeta \in \mathbb{R}^n_{+}$. The KKT conditions for the primal problem (\ref{eq:rho_quadrate_weights}) are: 
\begin{numcases}{}
\label{kkt:stationary}
l_i + \lambda (np_i - 1) + \eta - \zeta_i = 0, \forall 1 \leq i \leq n \\ 
\label{kkt:slack_1}
\lambda \left(2\rho + 1 - n \| p\|_2^2 \right)  = 0 \\ 
\label{kkt:primal_1}
n \| p\|_2^2 \leq 2\rho + 1  \\ 
\label{kkt:dual_1}
\lambda \geq 0 \\ 
\label{kkt:primal_2}
\sum_{i=1}^n p_i = 1 \\ 
\label{kkt:slack_2}
\zeta_i p_i = 0,  \forall 1 \leq i \leq n \\ 
\label{kkt:dual_2}
\zeta_i \geq 0, \forall 1 \leq i \leq n.
\end{numcases}  
We can see that the strong duality holds because the primal problem in (\ref{eq:rho_quadrate_weights}) satisfies the Slater's condition; therefore the KKT conditions are the necessary and sufficient conditions for the primal optimal solution. It follows from Equations (\ref{kkt:stationary}), (\ref{kkt:slack_2}), and (\ref{kkt:dual_2}) that:
\begin{align}
    \label{eq:p}
    \lambda n p_i = (-l_i -\eta)_{+} := \max \{ -l_i -\eta,0 \},
\end{align}
which, combined with Equation (\ref{kkt:primal_2}), implies that: 
\begin{align}
    \label{eq:eta_to_lam}
    n \lambda = \sum_{i=1}^n (-l_i - \eta)_{+}
\end{align}

From Equation (\ref{eq:eta_to_lam}), we have:
\begin{align}
    \label{eq:lam_to_eta}
    \eta &= \frac{-\sum_{i \in A} l_i - n \lambda}{ |A|},
\end{align}
where $A = \{ i: l_i + \eta \leq 0  \} $. Note that $|A| \geq 1$ because otherwise $p_i =0, \forall 1 \leq i \leq n$ which contradicts Equation (\ref{kkt:primal_2}). We then plug Equation (\ref{eq:lam_to_eta}) and (\ref{eq:p}) into Equation (\ref{kkt:primal_1}) to solve for $\lambda$. Note that $\|p\|_2^2$ is decreasing in $\lambda$, thus we can bisect to find the optimal $\lambda$ within its bound. We can easily obtain a bound on $\lambda$ from Equation (\ref{kkt:primal_1}): 
\begin{align*}
    0 \leq \lambda &\leq \max \left\{  \frac{-l_{min} + \sum_{i=1}^n l_i  }{\sqrt{1 + 2 \rho} - 1}, \frac{-l_{min} + l_{max}}{ \sqrt{1 + 2 \rho} }   \right \}, 
\end{align*}
where $l_{min} = \min_{1 \leq i \leq n} l_i$, and  $l_{max} = \max_{1 \leq i \leq n} l_i$.

The optimal distribution $\argmin_{p \in \mathcal{P}_{n,\rho}} \sum_{i=1}^n p_i l_i $ is not constant, but rather a function of $l$. Thus, its gradients with respect to some parameter $\psi$ must be computed from those of $l$. This becomes straightforward when we have solved $(p_i, \lambda, \eta)$ in terms of $l$ as in the results above:
 \begin{align*}
 \begin{cases}
    \label{eq:dr_gradient}
    \frac{\partial p_i}{\partial \psi}  = \frac{-1}{n \lambda^2} (-l_i - \eta) \frac{\partial \lambda}{\partial \psi}  + \frac{1}{n \lambda}(-\frac{\partial l_i}{\partial \psi} -\frac{\partial \eta}{\partial \psi}) \\
    |A| \frac{\partial \eta}{\partial \psi} =  -\sum_{i \in S} \frac{\partial l_i}{\partial \psi} - n \frac{\partial \lambda}{\partial \psi} \\ 
    \sum_{i \in S} p_i \frac{\partial p_i}{\partial \psi} = 0.      
 \end{cases}
 \end{align*}

\section*{Appendix B.2}
In this appendix, we present the details of the bisection search in Algorithm \ref{alg:bisect} for computing the $\rho$-robust distributions for line 3 of Algorithm 1 in the main text. 

\setcounter{algocf}{1}
\begin{algorithm}
\DontPrintSemicolon
\KwInput{$p(\lambda)$ computed in Proposition 1 in the main text, $\epsilon \geq 0$}
    $\lambda_{min} = 0$ \;
    
    $\lambda_{max} = \max \left\{  \frac{-l_{min} + \sum_{i=1}^n l_i  }{\sqrt{1 + 2 \rho} - 1}, \frac{-l_{min} + l_{max}}{ \sqrt{1 + 2 \rho} }   \right \}$ \; 
    
    $\lambda = \lambda_{min}$ \; 

    \While{$\lambda_{max} - \lambda_{min} > \epsilon$}
    {
        $\lambda = \frac{1}{2}(\lambda_{max} + \lambda_{min})$ ;\ 
        
        \If{$n \| p(\lambda)\|_2^2 > 2\rho + 1$}
        {
            $\lambda_{min} = \lambda$ 
        }
        
        \Else
        {
            $\lambda_{max} = \lambda$
        }
    }
\KwOutput{$\lambda, p(\lambda)$}
\caption{Bisection search}
\label{alg:bisect}
\end{algorithm}

\section*{Appendix B.3}
For the details of derivation for posterior sampling (a.k.a Thompson sampling), see Appendix A of \cite{nipsMuandetFDS12}. 


\section*{Appendix C}
In this appendix, we provide some more experimental results of DRBQO on synthetic and real-world problems. 

\setcounter{figure}{5}
\begin{figure}[H]
    \centering
    \includegraphics[scale=0.62]{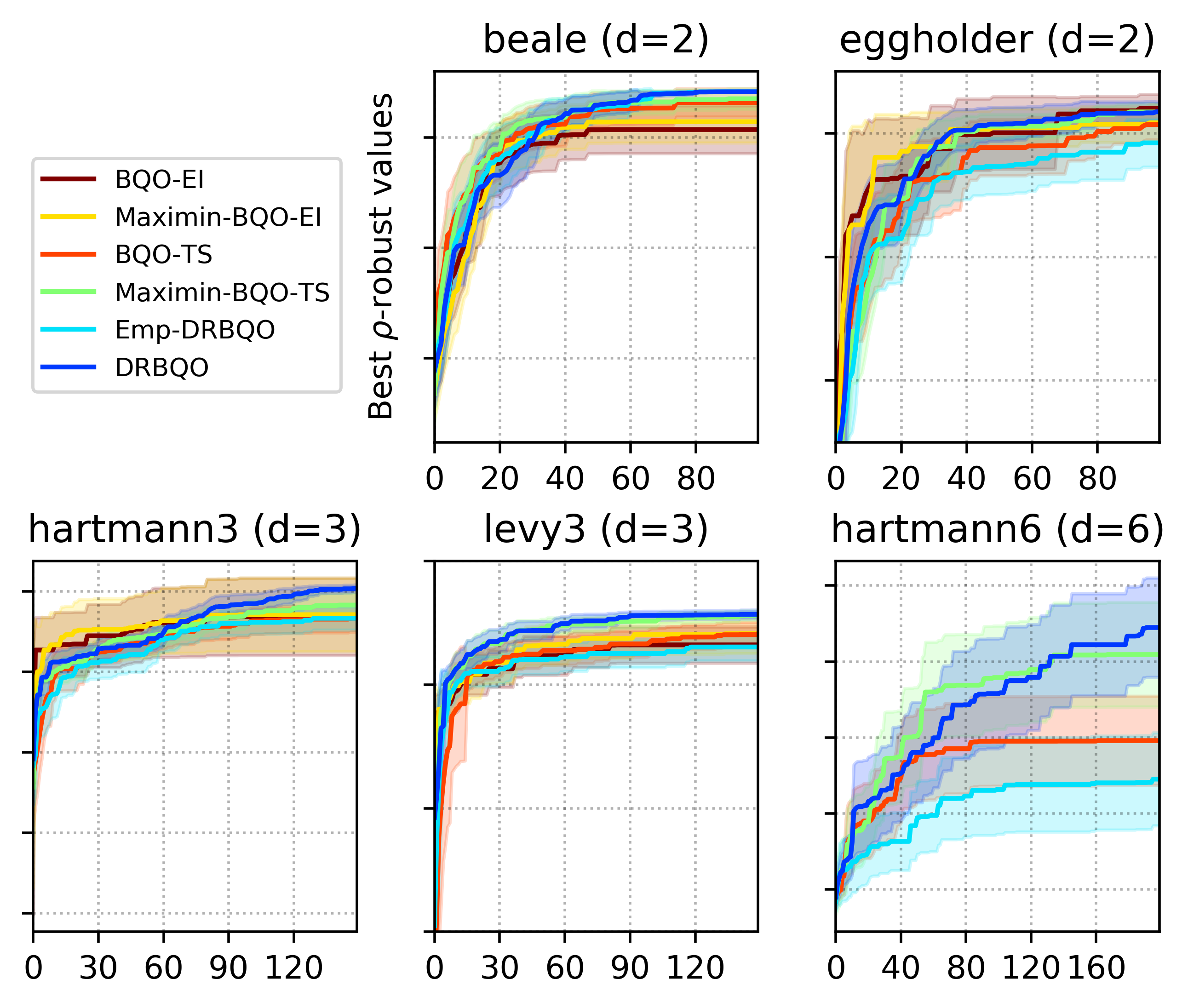}
    \caption{The performance of DRBQO and the baselines on the expected reformulation of various synthetic functions. Here we use $n=10$ and the best $\rho$ values are calculated with $\rho=1.0$. DRBQO achieves higher $\rho$-robust values than the BQO baselines in almost all the tested functions.}
    \label{fig:synthetic_various}
\end{figure}

\textbf{Synthetic functions}. The task in this experiment is to maximize $\mathbb{E}_{w \in \mathcal{N}(0,1)}[f(x,w)]$ where $f$ is a standard synthetic function such as Beale, Eggholder, Hartmann and Levy, $x$ is normalized to the unit cube and $f(x,w) := f(x+w)$. The performance metric used in this experiment is the $\rho$-robust values $\min_{P \in \mathcal{P}_{n,\rho}} \mathbb{E}_{P(w)}[f(x,w)]$. Here we use $n=10$ and $\rho=1.0$. We repeat the experiment $30$ times and report the average mean and the $96\%$ confidence interval for each evaluation
metric. The result is presented in Figure \ref{fig:synthetic_various}. The result shows that DRBQO achieves higher $\rho$-robust values than the baseline methods in all these functions except that in EggHolder function, DRBQO is compatable with BQO-EI but outperforms the other algorithms.

\textbf{Cross-validation hyperparameter tuning for SVM}. 
We use glass and connectionist bench classification datasets from UCI machine learning repository. \footnote{http://archive.ics.uci.edu/ml} The glass dataset contains $214$ samples describing glass properties in $10$ features. The task associated with the glass dataset is to classify an example into one of $7$ classes. The connectionist bench dataset contains $208$ samples each of which has $60$ attributes. The task in the connectionist bench dataset is to classify whether sonar signals bounced off a metal cylinder or a roughly cylindrical rock. Each of the datasets is split into the training and test sets with the ratio of $80:20$. The training set is further split into $n=5$ folds for this experiment. 

\begin{figure}[H]
    \centering
    \includegraphics[scale=0.6]{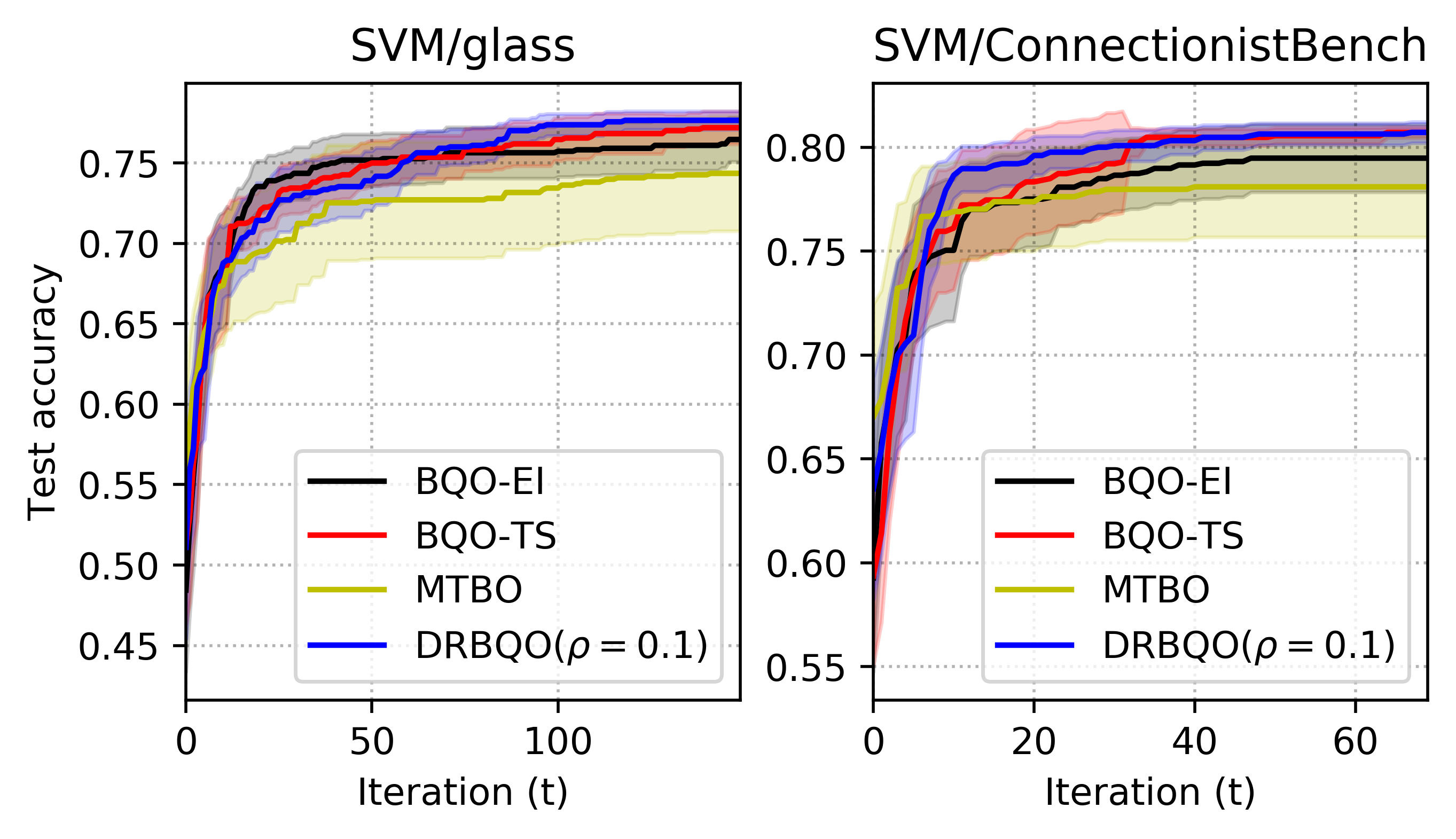}
    \caption{The test classification accuracy of SVM on glass and connectionist bench dataset tuned by DRBQO and the BQO baselines. In this example, we use $n=5$. }
    \label{fig:svm}
\end{figure}

Support vector machine (SVM) is a simple machine learning algorithm for classification problems. SVMs with RBF kernels have two hyperparameters: the misclassification trade-off $C$ and the RBF hyperparameter $\gamma$. We tuned these two hyperparameters in this example.  

The performance metric for this experiment is the classification accuracy of SVM in the test set.  We repeat the experiment $30$ times and report the average mean and the $96\%$ confidence interval for each evaluation
metric. The result is presented in Figure \ref{fig:svm}. In this example, DRBQO outperforms the baselines.

\end{document}